%% file: main.tex
\documentclass[sts,
	reqno,
]{imsart}

\usepackage[top=2.5cm,bottom=2.5cm,left=2.5cm,right=2.5cm,includeheadfoot,asymmetric]{geometry}

\usepackage[utf8]{inputenc}

\usepackage{amsthm, amssymb, amsmath, mathtools}
\usepackage{bbm}
\usepackage[toc,page]{appendix}
\usepackage{euscript}
\usepackage{subfig}
\newtheorem{theorem}{Theorem}
\newtheorem{lemma}{Lemma}
\newtheorem{proposition}{Proposition}
\newtheorem{definition}{Definition}

\usepackage{color}
\usepackage{tikz}
\usetikzlibrary{arrows}

\usepackage[shortlabels]{enumitem}
\usepackage{url}
\usepackage{etex}
\usepackage{memhfixc}
\usepackage{bbold}%
\usepackage{cases}
\definecolor{labelkey}{rgb}{0.0, 0.8, 0.3}

\usepackage{algorithm}
\usepackage{xfrac}
\usepackage{wrapfig}
\usepackage{natbib,mwe,cancel,float,dsfont}

\newtheorem{fact}{Fact}

\makeatletter
\newtheorem*{rep@theorem}{\rep@title}
\newcommand{\newreptheorem}[2]{%
	\newenvironment{rep#1}[1]{%
		\def\rep@title{\textbf{#2} \ref{##1}}%
		\begin{rep@theorem}}%
		{\end{rep@theorem}}}
\makeatother
\newreptheorem{theorem}{Theorem}
\newreptheorem{lemma}{Lemma}
\newreptheorem{proposition}{Proposition}
\newreptheorem{assumption}{Assumption}
\newreptheorem{corollary}{Corollary}

\usepackage{style}
\usepackage{nicematrix}

\numberwithin{equation}{section}
\usepackage{hyperref}
\hypersetup{
  colorlinks = true,
  urlcolor = blue, 
  linkcolor = blue,
  citecolor = blue}
 
\usepackage{cleveref}
  
\renewcommand{\abstractname}{Abstract}

\usepackage{xcolor,colortbl}
\usepackage{comment}

\newtheorem{corollary}[theorem]{Corollary}

\begin{document}

\begin{frontmatter}

	\title{Active Fourier Auditor for Estimating\\ Distributional Properties of ML Models}
	\runtitle{Active Fourier Auditor}
		\author{\textbf{Ayoub Ajarra} \hfill ayoub.ajarra@inria.fr\\ 
        \'Equipe Scool, Univ. Lille, Inria, CNRS, Centrale Lille, UMR 9189- CRIStAL \\
	\textbf{Bishwamittra Ghosh} \hfill bghosh@u.nus.edu\\
 Max Planck Institute for Software Systems, Saarbr\"ucken, Germany\\
	\textbf{Debabrota Basu} \hfill debabrota.basu@inria.fr \\
 \'Equipe Scool, Univ. Lille, Inria, CNRS, Centrale Lille, UMR 9189- CRIStAL, France }
	\address{{\'Equipe Scool} \\
		{Univ. Lille, Inria} \\
		{UMR 9189-CRIStAL, CNRS, Centrale Lille}\\
		{Lille, France}}
	\address{{Max Planck Institute for Software Systems} \\
		{Saarbr\"ucken, Germany}}
	\address{{\'Equipe Scool} \\
		{Univ. Lille, Inria} \\
		{UMR 9189 - CRIStAL, CNRS, Centrale Lille}\\
		{Lille, France}}
	\runauthor{Ajarra, Ghosh and Basu}
\begin{abstract}  
	\small{\input{abstract.tex}}
\end{abstract}
\end{frontmatter}

\newpage
\tableofcontents
\newpage

\section{Introduction}

As Machine Learning (ML) systems are pervasively being deployed in high-stake applications, mitigating discrimination and guaranteeing reliability are critical to ensure the safe pre and post-deployment of ML~\citep{euaiact}. These issues are addressed in the growing subfield of ML, i.e. trustworthy or responsible ML~\citep{rasheed2022explainable,li2023trustworthy}, in terms of robustness and fairness of ML models. Robustness quantifies how stable are a model's predictions under perturbation of its inputs~\citep{robust2011shie,kumar2020adversarial}. Fairness~\citep{fairness2012dwork,fairmlbook} seeks to address discrimination in predictions both at the individual level and across groups. Thus, AI regulations, such as the European Union AI Act~\citep{euaiact}, increasingly suggest certifying different model properties, such as robustness, fairness, and privacy, for a safe integration of of ML in high-risk applications. Thus, estimating these model properties under minimum interactions with the models has become a central question in algorithmic auditing~\citep{raji2020closing,wilson2021building,metaxa2021auditing,Reconstruct&Audit:Yan}.
\begin{example}
    \normalfont
    Following~\citep[Example 1]{ghosh2021justicia}, let us consider an ML model that predicts who is eligible to get medical insurance given a sensitive feature `age', and two non-sensitive features `income' and `health'. Owing to historical bias in the training data, the model, i.e. an explainable decision tree, discriminates against the `elderly' population by denying their health insurance and favors the `young' population. 
    Hence, an auditor would realize that the model does not satisfy \textit{group fairness} since the difference in the probability of approving health insurance between the elderly and the young is large. In addition, the model violates \textit{individual fairness}, where perturbing the feature `age' from elderly to young increases the probability of insurance. Further, the model lacks \textit{robustness} if perturbing any feature by an infinitesimal quantity flips the prediction.
\end{example}

\textbf{Related Work: ML Auditing}. Towards trustworthy ML, several methods have been proposed to ally audit an ML model by estimating different \textit{distributional properties} of it, such as fairness and robustness, where the model hyper-property has to be assessed against the distribution of inputs. 
A stream of work focuses on property verification that verifies whether these properties are violated above a pre-determined threshold~\citep{mlverification:shafi,IFverification:jhon,pacverification:Mutreja,BBverification:Herman,audit:Kearns}. 
Thus, we focus on estimating these properties instead of a `yes/no' answer, which is a harder problem than verification~\citep{mlverification:shafi}. On estimating distributional properties, \cite{BayesAudit:Neiswanger} proposed a Bayesian approach for estimating properties of black-box optimizers and required a prior distribution of models. \cite{wang2022beyond} studies simpler distributional properties, e.g. the mean, the median, and the trimmed mean defined as a conditional expectation, using offline and interactive algorithms. \cite{Reconstruct&Audit:Yan} considered a frequentist approach for estimating group fairness but assumed the knowledge of the model class and a finite hypothesis class under audit. These assumptions are violated if we do not know the model type and can be challenging for complex models, e.g. deep neural networks. \cite{albarghouthi2017fairsquare,ghosh2021justicia} considered finite models for estimating group fairness w.r.t.\ the features distribution, and \cite{ghosh2022algorithmic} further narrowed down to linear models. 
Therefore, we identify the following limitations of the existing methods in ML auditing. (1) \textbf{Property-specific auditing:} most methods considered a property-specific tailored approach to audit ML systems, for example either robustness~\citep{cohen2019certified,  salman2019provably}, group fairness~\citep{albarghouthi2017fairsquare,ghosh2021justicia}, or individual fairness~\citep{IFverification:jhon}. (2) \textbf{Model-specific auditing:} all the methods considered a prior knowledge about the ML model~\citep{BayesAudit:Neiswanger,ghosh2021justicia,ghosh2022algorithmic,Reconstruct&Audit:Yan}, or a white-box access to it~\citep{cohen2019certified, salman2019provably}. These are unavailable in practical systems such as API-based ML. Therefore, our research question is: \textit{Can we design a unified ML auditor for black-box systems for estimating a set of distributional properties including robustness and fairness?}

\begin{figure}[t!]
\centering
\begin{adjustbox}{width=0.7\textwidth} 
\begin{circuitikz}
\tikzstyle{every node}=[font=\normalsize]
\draw  (3.5,13.75) rectangle (7.25,11.75);
\node [font=\normalsize] at (5.25,12.75) {Black-box model};
\node [font=\normalsize] at (5,11.5) {\textbf{}};
\draw  (10.25,13.5) rectangle (14.25,12.25);
\node [font=\normalsize] at (12.25,13) {Sampler};
\draw  (10.25,15.75) rectangle (14.25,14.5);
\node [font=\normalsize] at (12,11) {\textbf{}};
\node [font=\normalsize] at (12.25,15) {Fourier Coefficients};
\draw [->, >=Stealth] (10.25,13.25) -- (7.25,13.25);
\draw [->, >=Stealth] (7.25,12.5) -- (10.25,12.5);
\draw [ color={rgb,255:red,0; green,0; blue,200} , line width=0.8pt , rounded corners = 28.8, , dashed] (9.75,16.25) rectangle  (14.75,11.5);
\draw [->, >=Stealth] (14.25,15.25) -- (16.25,15.25);
\node [font=\normalsize] at (5.25,12.25) {};
\draw [, dashed] (5.5,15.25) ellipse (2.25cm and 0.75cm);
\node [font=\normalsize] at (5.5,15.25) {Data Distribution};
\draw [<->, >=Stealth] (12.25,14.5) -- (12.25,13.5);
\draw [->, >=Stealth, dashed] (5.5,14.5) -- (5.5,13.75);
\draw [->, >=Stealth, dashed] (7.75,15) -- (10.25,13.5);
\node [font=\normalsize,right] at (16.25,15.75) {Robustness};
\node [font=\normalsize,right] at (16.25,15.25) {Individual Fairness};
\node [font=\normalsize,right] at (16.25,14.75) {Group Fairness};
\node [font=\normalsize, color={rgb,255:red,0; green,0; blue,200}] at (12,11.75) {\textbf{    AFA}};
\node [font=\normalsize] at (8.5,12.25) {Prediction};
\node [font=\normalsize] at (8.5,13.5) {Input sample};
\end{circuitikz}
\end{adjustbox}\vspace*{-1em}
\caption{A schematic of \afa.}\label{fig:afa}\vspace*{-1.5em} 
\end{figure}
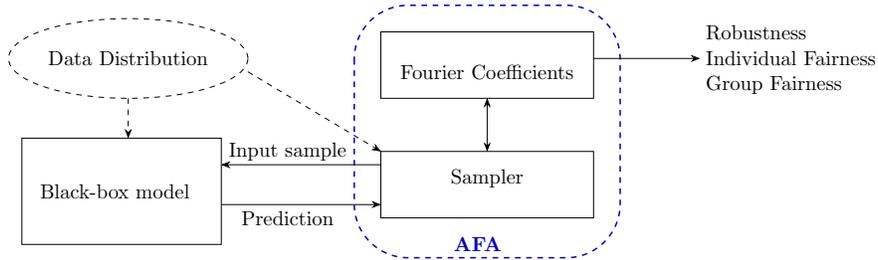
\textbf{Contributions.} We propose a framework, namely \afa{} (\textbf{A}ctive \textbf{F}ourier \textbf{A}uditor), which is an ML auditor based on the Fourier approximation of a black-box ML model (Figure~\ref{fig:afa}). We observe that existing black-box ML auditors work in two steps: \textit{the model reconstruction step}, where they reconstruct a model completely, and \textit{the estimation step}, where they put an estimator on top of it~\citep{Reconstruct&Audit:Yan}. We propose a model-agnostic strategy that does not need to reconstruct the model completely. In particular, for any ML model admitting a Fourier expansion, we compute the significant Fourier coefficients of a model accepting categorical input distributions such that they are enough to estimate different distributional properties such as robustness, individual fairness, and group fairness. Our contributions are:
\begin{itemize}[leftmargin=*]
    \item \textit{Formalism.} For any bounded output model (e.g. all classifiers), we theoretically reduce the estimation of robustness, individual fairness, and group fairness in terms of the Fourier coefficients of the model. The key idea is based on influence functions, which capture how much a model output changes due to a change in input variables and can be computed via Fourier coefficients (Section~\ref{sec:methods}). We propose two types of influence functions for each of these properties that unifies robustness and individual fairness auditing while put group fairness in a distinct class.

    \item \textit{Algorithm.}  In \afa{}, we integrate Goldreich-Levin algorithm~\citep{Goldreich:Goldreich,Goldreich:Mansour} to efficiently compute the significant Fourier coefficients of the ML model, which are enough to compute the corresponding properties. \afa{} yields a probably approximately correct (PAC) estimation of distributional properties. We propose a dynamic version of Goldreich-Levin to accelerate the computations.

    \item \textit{Theoretical Sample Complexity.} We show that our algorithm requires $\Tilde{\cO} \Big(\frac{1}{\epsilon} \sqrt{\log \frac{1}{\delta}}  \Big)$ samples to yield $(\epsilon,\delta)$ estimate of robustness and individual fairness, while it needs 
    $\Tilde{\cO} \left(\frac{1}{ \epsilon^2} \log \frac{1}{\delta}\right)$ samples to audit group fairness. We further derive a lower bound on the sample complexity of $(\epsilon,\delta)$-auditing of group fairness to be $\Tilde{\Omega}(\frac{\delta}{\epsilon^2} )$. Further, for group fairness, we prove that \afa{} is manipulation-proof under perturbation of $2^{n-1}$ Fourier coefficients.
    
    \item \textit{Experimental Results.} We numerically test the performance of \afa{} to estimate the three properties of different types of models. The results show that \afa{} achieves lower estimation error while estimating robustness and individual fairness across perturbation levels. Compared to existing group fairness auditors, {\afa{}} not only achieves lower estimation error but also incurs lower computation time across models and the number of samples.
\end{itemize}

\section{Background}\label{sec:background}
Before proceeding to the contributions, we discuss the three statistical properties of ML models that we study, i.e. robustness, individual fairness, and group fairness. We also discuss basics of Fourier analysis that we leverage to design \afa.

\noindent\textbf{Notations.}  Here, $x$ represents a scalar, and $\bx$ represents a vector. $\cX$ is a set. We denote $\Iintv{1,n}$ as the set $\{1,\ldots, n\}$. We denote the power set of $\cX$ by $\cP(\cX)$.

\noindent\textbf{Properties of ML Models.} A Machine Learning (ML) model $h$ is a deterministic or probabilistic mapping from an $n$-dimensional input domain of features (or covariates) $\cX$ to set of labels (or response variables or outcomes) $\cY$.
For example, for Boolean features $\cX \triangleq \{-1,1\}^n$, and for categorical features, $\cX \triangleq [K]^n$. For binary classifiers, $\cY \triangleq \{0,1\}$. 

We assume to have only \textit{black-box access to $h$}, i.e. we send queries from a data-generating distribution and collect only the labels predicted by $h.$ The dataset on which $h$ is tested is sampled from a data-generating distribution $\cD_{\cX,\cY}$ over $\cX \times \cY$, which has a marginal distribution $\cD$ over $\cX$.




\textit{We aim to audit a \emph{distributional} (aka global) \emph{property} $\mu: \cH \times \cD_{\cX,\cY} \rightarrow \R$ of an ML model $h: \cX \rightarrow \cY$ belonging to an \textit{unknown} model class $\cH$ while having only black-box access to $h$.}

Hereafter, \textit{we develop the methodology for binary classifiers and Boolean features}. Later, we discuss approaches to extend the proposed methodology to categorical features and multi-class classifiers, and corresponding experimental results. In this paper, we study three properties of ML models, i.e. robustness ($\robuste$), individual fairness ($\IFaire$), and group fairness ($\GF$), which are defined below. 

\noindent\textbf{Robustness} is the ability of a model $h$ to generate same output against a given input and its perturbed (or noisy) version. Robustness has been central to sub-fields of AI, e.g. safe RL~\citep{garcia2015comprehensive}, adversarial ML~\citep{kurakin2016adversarial,biggio2018wild}, and gained attention for safety-critical deployment of AI.

\begin{definition}[Robustness]\label{def:robustness}
Given a model $h$ and a perturbation mechanism $\Gamma$ of input $x \in \cX$, robustness of $h$ is
$\robust{h} \triangleq \prob_{\mathbf{x} \sim \mathcal{D},~\mathbf{y} \sim \Gamma(x)}[h(\mathbf{x}) \neq h(\mathbf{y})]$.
\end{definition}

Examples of perturbation mechanisms include Binary feature flipping  $N_{\rho}(\bx) \triangleq \{\mathbf{x'} \mid \forall i \in [n], \mathbf{x'}_i=\mathbf{x}_i \times \mathrm{Bernoulli}(\rho)\}$~\citep{Boolean:Donnell}, 
Gaussian perturbation  $N_{\rho}(x) \triangleq \{\mathbf{x'} \mid \mathbf{x'}=\mathbf{x}+\mathbf{\epsilon} \text{ where } \epsilon \sim \mathrm{Normal}(0, \rho^2 I) \} $~\citep{cohen2019certified}, among others.

In trustworthy and responsible AI, another prevalent concern about deploying ML models is bias in their predictions. This has led to the study of different fairness metrics, their auditing algorithms, and algorithms to enhance fairness~\citep{mehrabi2021survey,fairmlbook}. There are two categories of fairness measures~\citep{fairmlbook}. The first is the \textbf{individual fairness} that aims to ensure that individuals with similar features should obtain similar predictions~\citep{fairness2012dwork}. 

\begin{definition}[Individual Fairness]\label{def:ifair}
For a model $h$ and a neighbourhood $\Gamma(x)$ of a $x \in \cX$, the individual fairness discrepancy of $h$ is $\IFair{h} \triangleq \prob_{{ \mathbf{x} \sim \mathcal{D},~\mathbf{y} \sim \Gamma(\bx)}}{\mathbb{P}}[h(\mathbf{x}) \neq h( \mathbf{y})]$.
\end{definition}

The neighborhood $\Gamma(x)$ is commonly defined as the points around $x$ which are at a distance less than $\rho \geq 0$ w.r.t. a pre-defined metric. The metric depends on the application of choice and the input data~\citep{mehrabi2021survey}. 
IF of a model measures its capacity to yield similar predictions for similar input features of individuals~\citep{fairness2012dwork,IF2016friedler}. The similarity between individuals are measured with different metrics. 
Let $d_{\mathcal{X}}$ and $d_{\mathcal{Y}}$ be the metrics for the metric spaces of input ($\cX$) and predictions ($\cY$), respectively. 

\textit{A model $h$ satisfies ($\epsilon,\epsilon'$)-IF if $ d_{\mathcal{X}}(\bx,\bx') \leq \epsilon$ implies $d_{\mathcal{Y}}(h(\bx),h(\bx')) \leq \epsilon'$ for all $(\bx, \bx') \in\mathcal{X}^2$}~\citep{IF2016friedler}. 
For Boolean features and binary classifiers, the natural candidate for $d_{\mathcal{X}}$ and $d_{\mathcal{Y}}$ is the \textit{Hamming distance}. This measures the difference between vectors $\mathbf{x}$ and $\mathbf{x}'$ by counting the number of differing elements. Thus, $d_{\mathcal{X}}( \mathbf{x}, \mathbf{x}') \leq l$ means that $\mathbf{x}'$ has $l$ different bits than $\mathbf{x}$. As auditors, we are interested in measuring how much the Hamming distance between outcomes of $\bx$ and $\bx'$, i.e. $\epsilon'$. 
However, since the data-generation process and the models might be stochastic, we take a stochastic view and use a perturbation mechanism that defines a neighborhood around each input sample.

\textbf{Group fairness} is the other category of fairness measures that considers the input to be generated from multiple protected groups (or sub-populations), and we want to remove discrimination in predictions across these protected groups~\citep{mehrabi2021survey}.
Specifically, we focus on \textit{Statistical Parity (SP)}~\citep{feldman2015certifying, fairness2012dwork} as our measure of deviation from group fairness. For simplicity, we discuss SP for two groups, but we can also generalize it to multiple groups.
\begin{definition}[Statistical Parity]\label{def:gfair}
 The statistical parity of $h$ is $\GFair{h} \triangleq |\prob_{\mathbf{x} \sim \mathcal{D}}[h(\mathbf{x}) = 1 |  x_A = 1] -  \prob_{\mathbf{x} \sim \mathcal{D}}[h(\mathbf{x}) = 1 |  x_A = - 1]|$, where $x_A$ is the binary sensitive attribute.
\end{definition}

\textit{In \afa{}, we use techniques of Fourier analysis to design one computational scheme for simultaneously estimating these three properties of an ML model.}

\noindent\textbf{A Primer on Fourier Analysis.}
Designing \afa{} is motivated by the Fourier expansion of Boolean functions. Fourier coefficients are distribution-dependent components that capture key information about the distribution's properties. This study was initially addressed by \citep{Boolean:Donnell}, who focused on the uniform distribution. Later, \citep{Fourier:Heidari} generalized this result to arbitrary distributions, which we leverage further.
\begin{proposition}[\cite{Fourier:Heidari}]\label{prop:MohsenGram}
   There exists a set of orthonormal parity functions $\{\psi_S\}_{S \subseteq [n]}$ such that any function $h: \{-1,1\}^n \to \{-1,1\}$ is decomposed as 
   \begin{align}\label{eq:expansion}
       h(x) = \sum\nolimits_{S \subseteq[n]} \hat{h}(S) \psi_S(\mathbf{x})~\text{ for any }~x \sim \mathcal{D}.
   \end{align}
   The Fourier coefficients 
   $\hat{h}(S)\triangleq \E_{\mathbf{x} \sim \mathcal{D}}[h(X) \psi_S(\mathbf{x})]$ are unique for all $S \subseteq [n]$.
\end{proposition}

\begin{example}
    Let us consider $h$ to be the XOR function on $\bx \in \{-1, 1\}^2$. This means that $h(-1, -1) = h(1, 1) = 0$ and $h(1, -1) = h(-1, 1) = 1$. The Fourier representation of $h(\bx) = 0.5 + 0.5x_1 + 0.5 x_2 -0.5 x_1x_2$, when $\bx$ is sampled from a uniform distribution on $\{-1, 1\}^2$. 
\end{example}
\begin{example}\label{example1}
Suppose random variables $X_1 $ and $X_2 $ are drawn i.i.d. from the standard normal distribution $\mathcal{N}(0,1)$~\citep{Fourier:Heidari}. Define another random variable $X_3 $ as $X_3 = X_1 X_2$. It can be verified that the Gram-Schmidt basis of XOR of $X_1, X_2, X_3$ has four zero coefficients, i.e. the sets including $X_3$ do not influence the outcomes. This is because $X_3$'s information is encoded in $X_1$ and $X_2$ jointly.
\begin{table}
\centering
\caption{Example~\ref{example1}}
\resizebox{0.6\columnwidth}{!}{
\begin{tabular}{c|cccccccc}
\toprule
$S$       & $\emptyset$ & $\{1\}$ & $\{2\}$ &$\{1,2\}$ & $\{3\}$ &  $\{1,3\}$ & $\{2,3\}$ & $\{1,2,3\}$ \\ \midrule
$\chi_S$  & $1$ &$x_1$& $x_2$  & $x_1 x_2$  &       $x_3$    & $x_1 x_3$ & $x_2 x_3$ &  $x_1 x_2 x_3$  \\
$\psi_S$  &$1$& $x_1$   &  $x_2$  & $x_1 x_2$   &   $0$ & $0$ & $0$  & $0$ \\ \bottomrule
\end{tabular}}
\end{table}
\end{example}

\noindent\textbf{Influence functions.} To estimate the properties of interest, we use a tool from Fourier analysis, i.e. \textit{influence functions}~\citep{Boolean:Donnell}. They measure how changing an input changes the output of a model. Different influence functions are widely used in statistics, e.g. to design robust estimators~\citep{mathieu2022bandits}, and ML, e.g. to find important features~\citep{Fourier:Heidari}, to evaluate how features induce bias~\citep{ghosh2021justicia}, to explain contribution of datapoints on predictions~\citep{Fourier:Ilyas}. Here, we use them to estimate model properties.

\begin{definition}[Influence functions]\label{influence}
    If $\Gamma$ is a transformation of an input $\bx \in \mathcal{X}$, the influence function is defined as
    $\mathrm{Inf}_{\Gamma}(h) \triangleq \prob_{\mathbf{x} \sim \mathcal{D}}[h(\bx) \neq h(\Gamma(\bx))].$
    $\mathrm{Inf}_{\Gamma}(h)$ is called \textit{deterministic} if the transformation $\Gamma$ is deterministic, and \textit{randomized} if $\Gamma$ randomized.
\end{definition}

In general, deterministic influence functions are used in Boolean function analysis~\citep{Boolean:Donnell}. 
In contrast, in Section~\ref{sec:methods}, we express robustness, individual fairness, and group fairness with randomized influence functions. We also show that the influence functions can be computed using the Fourier coefficients of the model under audit (Equation~\eqref{eq:expansion}).

\section{Active Fourier Auditor}\label{sec:methods}
In the black-box setting, the access to the model $h$ is limited by the query oracle, accessible to the auditor. The auditor's objective is to estimate the property $\mu$ through interaction with this oracle. The definition of the property estimator relies on the information made available to the auditor during this interaction. In the context of auditing with model reconstruction~\citep{Reconstruct&Audit:Yan}, the auditor is denoted as $\hat{\mu}: \mathcal{H} \times \cB \to \mathbb{R}$. Here, the auditor has access to an unlabeled pool and applies active learning techniques (e.g. CAL algorithm) to query samples. This process uses the additional information given by the hypothesis class where the model $h$ lives. Following the reconstruction phase, the auditor has an approximate model $\hat{h}$ of true model $h$, enabling estimation of the property via plug-in estimator $\hat{\mu}(\hat{h})$.

Now, we present a novel non-parametric black-box auditor \afa{} that assumes no knowledge of the model class and the data-generating distribution. Unlike the full model-reconstruction-based auditors, \afa{} uses Fourier expansion and adaptive queries to estimate the robustness, Individual Fairness (IF), and Group Fairness (GF) properties of a model $h$. In this setting, the auditor is defined as $\hat{\mu}: \cF_{\mu} \times \cB \to \mathbb{R}$, where $\cF_{\mu}$ represents the set of Fourier coefficients upon which the property $\mu$ depends.
First, we show that property estimation with model reconstruction always incurs higher error. Then, we show that robustness, IF, and GF for binary classifiers can be computed using Fourier coefficients of $h$. Finally, we compute the Fourier coefficients and thus, estimate the properties at once (Algorithm~\ref{algo:AFA}). 
We begin by defining a PAC-agnostic auditor that we realise with \afa.
\begin{definition}[PAC-agnostic auditor]\label{pacverifier}
    Let $\mu$ be a computable distributional property of model $h$. An algorithm $\cA$ is a \textit{PAC-agnostic auditor} if for any $\epsilon, \delta \in (0,1)$, there exists a function $m(\epsilon, \delta)$ such that $\forall m \geq m(\epsilon, \delta)$ samples drawn from $\mathcal{D}$, it outputs an estimate $\hat{\mu}_m$ satisfying $\prob(|\hat{\mu}_m - \mu| \leq \epsilon) \geq 1- \delta$.
\end{definition}

\noindent\textbf{Remark.} $\mu(h)$ is a \emph{computable} property if there exists a (randomized) algorithm, such that when given access to (black-box) queries, it outputs a PAC estimate of the property $\mu(h)$~\citep{audit:Kearns}. Any distributional property, including robustness, individual fairness and group fairness, is computable given the existence of the uniform estimator.

\subsection{The Cost of Reconstruction}

The naive way to estimate a model property is to reconstruct the model and then use a plug-in estimator~\citep{Reconstruct&Audit:Yan}. However, this requires an exact knowledge of the model class and comes with an additional cost of reconstructing the model before property estimation.
For group fairness, we show that the reconstruct-then-estimate approach induces significantly higher error than the reconstruction error, while the exact model reconstruction itself is NP-hard~\citep{jagielski2020high}.
\begin{proposition}\label{recost}  
If $\hat{h}$ is the reconstructed model from $h$, then
\small{\begin{align*}
           |\GFair{\hat{h}} - \GFair{h}| \leq  \min \Bigg\lbrace 1, \frac{\prob_{\bx \sim \cD}[\hat{h}(\bx) \neq h(\bx)]}{\min(\prob_{\mathbf{x} \sim \mathcal{D}}[\bx_A = 1], \prob_{\mathbf{x} \sim \mathcal{D}}[\bx_A = -1])}\Bigg\rbrace\,.
       \end{align*}}
\end{proposition}
Proposition~\ref{recost} connects the estimation error and the reconstruction error before plugging in the estimator. It also shows that to have a sensible estimation the reconstruction algorithm needs to achieve an error below the proportion of minority group, which can be significantly small requiring high sample complexity. The proof is deferred to Appendix~\ref{app:costreconstruct}. 
This motivates an approach that avoids model reconstruction by computing only the right components of the model expansion.
To capture the information relevant to estimating our properties of interest, we will represent them in terms of Fourier coefficients given in the model decomposition. Then we aim to adaptively estimate larger Fourier coefficients in contrast to model reconstruction method requiring to recovering all the Fourier coefficients. 
\subsection{Model Properties with Fourier Expansion}\label{sec:props_with_Fourier}

Throughout the rest of this paper, we denote by $\{\psi_S\}_{S \subseteq [n]}$ the basis derived from Proposition \ref{prop:MohsenGram}. In this section, we express the model properties of $h$ using its Fourier coefficients. The detailed proofs are deferred to Appendix~\ref{app:fourierpattern}.

\noindent\textbf{a. Robustness.} Robustness of a model $h$ measures its ability to maintain its performance when new data is corrupted. Auditing robustness requires a generative model to imitate the corruptions, which is modelled by the perturbation mechanism (Definition~\ref{def:robustness}). 
As we focus on the Boolean case, the worst case perturbation $\Gamma_{\rho}$ is the protocol of flipping vector coordinates with a probability $\rho$.
Specifically, a corrupted sample $\by$ is generated from $\bx$ such that for every component, we independently set $y_i = x_i$ with probability $\frac{1+\rho}{2}$ and $y_i = -x_i$ with probability $\frac{1-\rho}{2}$. 
This perturbation mechanism leads us to the $\rho$-flipping influence function.

\begin{definition}[$\rho$-flipping Influence Function]
    The $\rho$-flipping influence function of any model $h$ is defined as $\mathrm{Inf}_{\rho} (h) \triangleq \prob\nolimits_{\substack{ \mathbf{x} \sim \mathcal{D}, \mathbf{y} \sim \Gamma_{\rho}(\mathbf{x}) }}[h(\mathbf{x}) \neq h(\mathbf{y})]$.
\end{definition}
For a Boolean classifier, we further observe that $\mathrm{Inf}_{\rho} (h) = \E_{{\mathbf{x} \sim \mathcal{D},~\mathbf{y} \sim N_{\rho}(x)}}[h(\mathbf{x})h(\mathbf{y})]$. 
This allows us to show that the robustness of $h$ under $\Gamma_{\rho}$ perturbation is measured by $\rho$-flipping influence function, and thus, can be computed using Fourier coefficients of $h$.

\begin{proposition}\label{prop44}
Robustness of $h$ under the $\Gamma_{\rho}$ flipping perturbation is equivalent to the $\rho$-flipping influence function, and thus, can be expressed as
\begin{align}\label{eq:Rob_w_Fourier}
    \robust{h}  = \mathrm{Inf}_{\rho} (h) = \sum_{S \subseteq [n]} \rho^{|S|} \hat{h}(S)^2.
\end{align}
\end{proposition}

\noindent\textbf{b. Individual Fairness (IF).} To demonstrate the universality of our approach, we express IF with the model's Fourier coefficients. We consider the perturbation mechanism $\Gamma = \Gamma_{\rho,l}(\cdot)$ that independently flips uniformly $l$ vector coordinates with a probability $\frac{1+\rho}{2}$. Thus, we consider a neighbourhood with $\E_{\bx'\sim \Gamma_{\rho,l}(\bx)}[d_{\mathcal{X}}( \mathbf{x}, \mathbf{x}')] \leq \frac{1}{2}(1+\rho)l$ around each sample $\bx$ as the similar set of individuals. This perturbation mechanism leads us to the $(\rho,l)$-flipping influence function.

\begin{definition}[$(\rho,l)$-flipping influence function]
    The $(\rho,l)$-flipping influence function of any model $h$ is defined as $\mathrm{Inf}_{\rho,l} (h) = {\prob\nolimits_{\substack{\mathbf{x} \sim \mathcal{D}, \mathbf{y} \sim N_{\rho,l}(\mathbf{x})}}}[h(\mathbf{x}) \neq h(\mathbf{y})]$.
\end{definition}
We leverage $(\rho,l)$-flipping influence function to express IF of $h$ in terms of its Fourier coefficients (Proposition~\ref{prop47}). 
\begin{proposition}\label{prop47}
Individual fairness defined with respect to the $\Gamma_{\rho,l}$ perturbation is equivalent to the $(\rho,l)$-flipping influence function, and thus, can be expressed as
\begin{align}\label{eq:IF_w_Fourier}
    \IFair{h}  &= \mathrm{Inf}_{\rho,l} (h) = \sum_{S \subseteq [n]} \rho^{|S_l|} \hat{h}(S)^2\,,
\end{align}
where $S_l$ denotes the power sets for which $l$ features change.
\end{proposition}

\textbf{Unifying robustness and IF: The Characteristic Function.}
It is worth noting that IF is similar to robustness, differing only by a single degree of freedom, i.e. the number of flipped directions $l$. Specifically, from Equation~\eqref{eq:Rob_w_Fourier} and~\eqref{eq:IF_w_Fourier}, we observe that both the properties as $\mu(h) = \sum_{S \subseteq [n]} \texttt{char}(S, \mu_{\cdot}) \hat{h}(S)^2$, such that $\texttt{char}(S, \mu_{\mathrm{Rob}}) = \rho^{|S|}$, and $ \texttt{char}(S, \mu_{\mathrm{IFair}}) = \rho^{|S_l|}$. We call $\texttt{char}$ as \textit{the characteristic function of the property}.

\noindent\textbf{c. Group Fairness (GF).} Now, we focus on Group Fairness which aims to ensure similar predictions for different subgroups of population~\citep{fairmlbook}. We focus on Statistical Parity (SP) as the measure of deviation from GF~\citep{feldman2015certifying}. To quantify SP, we propose a novel membership influence function.

\begin{definition}[Membership influence function]\label{Inf}
    If $A$ denotes a sensitive feature, we define the membership influence function w.r.t. $A$ as the conditional probability $\mathrm{Inf}_{A} (h) \triangleq {\prob\nolimits_{\bx, \by \sim \cD}}\Big[h(\mathbf{x}) \neq h( \mathbf{y})\Big|x_A = 1, y_A = - 1 \Big]\,.$
\end{definition}

$\mathrm{Inf}_{A}(h)$ is the conditional probability of the change in the outcome of $h$ due to change in group membership of samples from $\cD$. In other words, it expresses the amount of independence between the outcome and group membership.

Note that the membership influence function is a randomised version of the deterministic influence function in~\citep{Boolean:Donnell}. If we denote the transformation of flipping membership, i.e. sensitive attribute of $\bx$, $f_A(\bx)$, the classical influence function is $\mathrm{Inf}^{\textit{det}}_{A} = \prob_{\bx \sim \cD}[h(\bx) \neq h(f_A(\bx))]$. The limitation of this deterministic function is that given $\bx \sim \cD$ the transformed vector $f_A(\bx)$ may not represent a sample from $\cD$. Thus, it fails to encode the information relevant to SP, whereas the proposed membership influence function does it correctly as shown below.

\begin{proposition}\label{groupfourier}
Statistical parity of $h$ w.r.t a sensitive attribute $A$ and distribution $\cD$ is the root of the second order polynomial $P_{\hat{h}}(X)$, i.e. $\alpha (1- \alpha) X^2 - \hat{h}(\emptyset)(1-2 \alpha) X - {\sum_{\substack{S \subseteq[n], S \ni A}}} \hat{h}(S)^2 - \frac{(  1 - \hat{h}^2(\emptyset)) }{2}$, where {$\alpha = \underset{\mathbf{x} \sim \mathcal{D}}{\mathbb{P}}[x_A = 1]$ and $\hat{h}(\emptyset)$ is the coefficient of empty set.}
\end{proposition}



\noindent\textbf{Summary of the Fourier Representation of Model Properties.} Robustness and individual fairness have the same Fourier pattern. They depend on all the Fourier coefficients of the model but differ only on their characteristic functions. In contrast, statistical parity of a sensitive feature $A$ depends only on the Fourier coefficient of that sensitive feature $\hat{h}(\{A\})$ and the Fourier coefficient of the empty set $\hat{h}(\emptyset)$.

\subsection{NP-hardness of Exact Computation} 
We have shown that the exact computation of robustness and individual fairness depends on all Fourier coefficients of the model. 
Since each Fourier coefficient of $h$ is given by
$\hat{h}(S) =  \underset{\mathbf{x} \sim \mathcal{D}}{\mathbb{E}}[h(\mathbf{x}) \psi_S(\mathbf{x})]$, exactly computing a single Fourier coefficient takes $\mathcal{O}(|\mathcal{X}|)$ time.
Additionally, the number of Fourier coefficients to compute to estimate robustness and individual fairness is exponential in the dimension of the input domain ($2^n$). 
Thus, exactly computing robustness and individual fairness requires $\mathcal{O}(2^n |\mathcal{X}|)$ time. This gives us an idea about the computational hardness of the exact estimation problem.  Now, we prove estimating large Fourier coefficients to be NP-complete.

\begin{theorem} \label{NP}
    Let $\cQ \triangleq \{\bx, h(\bx)\}$ be the set of input samples sent to $h$ and the predictions obtained.  Given $\tau \in \mathbb{R}_{\geq 0}$, exactly computing all the $\tau$-significant Fourier coefficients of $h$ is NP-complete.  
\end{theorem}
\noindent\textit{Proof Sketch.} For a set of queries $\mathcal{Q}$ and for each power set $S$, Fourier coefficient is given by $\hat{h}(S) = \frac{1}{|\mathcal{Q}|} \sum_{(x,h(x)) \in \mathcal{Q}}h(x) \psi_S(x) $. Maximizing the Fourier coefficient $|\hat{h}(S)|$ is equivalent to maximizing the agreement or disagreement between $h$ and the sign of $\psi_S$ for each truth assignment. 
Alternatively, maximizing $|\hat{h}(S)|$ is equivalent to finding a truth assignment that maximizes the number of true clauses in a CNF, where each clause is a disjunction of $h(x)$ and the sign of $\psi_S(x)$, and the CNF includes all such clauses for all $x\in\mathcal{Q}$. 
This is known as the Max2Sat (maximum two satisfiability) problem, which is known to be NP-complete. Hence, we conclude that finding large Fourier coefficients is also NP-complete. 
This result shows that the exact computation of the Fourier coefficients for our properties is NP-hard. This has motivated us to design \afa{}, which we later proved to be an $(\epsilon, \delta)$-PAC agnostic auditor.

\subsection{Algorithm: Active Fourier Auditor (\afa)}

We have shown that finding significant Fourier coefficients can be an NP-hard problem. In this section, we propose \afa{} (Algorithm~\ref{algo:AFA}) that takes as input a \textit{restricted access} of $q > 0$ queries from the data-generating distribution and requests labels from the black-box oracle of $h$ (Line 2). Those queries enable us to find the squares of significant Fourier coefficients and estimate them simultaneously. The list of the significant Fourier coefficients $L_h$ of the model $h$ contains both subsets and their estimated Fourier weights. We adopt a Goldreich-Levin (GL) algorithm based approach~\citep{Goldreich:Goldreich, Goldreich:Mansour} to find such list of significant Fourier coefficients (Figure~\ref{Fig:treeGL}). Since estimating the properties -- robustness, individual fairness and group fairness -- depend on estimating those Fourier coefficients, we plug in their computed estimates and output an $(\epsilon, \delta)$-PAC estimate of the properties (Line 4 and 5). 

\setlength{\textfloatsep}{4pt}
\begin{algorithm}[t!]
    \caption{Active Fourier Auditor (\afa)}\label{algo:AFA}
    \begin{algorithmic}[1]
    \STATE    \textbf{Input}: Sensitive attribute $A$, Query access to $h$, $\tau, \delta \in(0,1)$, $\epsilon \leftarrow \tau^2 / 4$
    \STATE $\{x_k,h(x_k)\} _{k \in [q]}\leftarrow$ \textsc{BlackboxQuery}$(h,q)$
        \STATE $L_h \leftarrow $ \textsc{GoldreichLevin}($h,q,\tau,\delta$) 
        \STATE $\hat{\mu}(h) \leftarrow \underset{S \in L_h}{\sum} \texttt{char}(\mu, S) \hat{h}(s)^2 $
        \STATE ${\hat{\mu}_{GF}(h)  \leftarrow P_{\hat{h}}^{-1}(0)}$
        \STATE \textbf{return} $\{\hat{\mu}_{RB},\hat{\mu}_{IF}, \hat{\mu}_{GF}\}$
    \end{algorithmic}
\end{algorithm}

\paragraph{Algorithmic Insights.} To compute the significant Fourier coefficients, we start with the power set. 
Now, we denote the subsets containing an element $i$ as $\cB_i({\mathcal{X}})$, and the subsets not containing $i$ as $\cB_{\neg i}({\mathcal{X}})$. 
Let $\Upsilon$ denote a trajectory starting from the set of all Fourier coefficients in the binary search tree of Fourier coefficients (Figure~\ref{fig:afa}). The question is that \textit{from the power set, how can we design a $\Upsilon$ to reach subsets of Fourier coefficients above a given threshold $\tau$?} 

In $\afa{}$, we dynamically create  ``buckets" of coefficients for this purpose. Each bucket $\cB^{S,k}$, represents a collection of power sets, such that $\cB^{S,k} \triangleq \{S \cup T \mid T \subseteq \{k+1, \dots, n\}\}$.
The corresponding weight is quantified by $ \cW^{S,k} \triangleq \sum_{T \subseteq \{k+1, \dots, n\}} \hat{h}(S \cup T)^2$. In this context, $\cW^{S,k}$ measures the total contribution of the Fourier coefficients associated with the elements in the bucket $\cB^{S,k}$. 
The bucket is initialized at $\cB^{\emptyset,0}$, which represents the weight of the power set of $\llbracket 1,n \rrbracket$. By Parseval's identity, we know that the weight of the power set is $1$, i.e. $ {\sum_{S \in \cP(\cX)}} \hat{h}(S)^2 = 1$. The bucket $\cB^{S,k}$ is then split into two buckets of the same cardinal: $\cB^{S,k-1}$ and $\cB^{S \cup \{k+1\},k+1}$. We then estimate the weight of each bucket by sending black-box queries to the model $h$.
The algorithm discards the bucket whose weight is below the threshold. When all the buckets collected at a round consist of exactly one element each, i.e. we reach the leaves, the algorithm halts and the buckets collected in this process are subsets of $\llbracket 1, n \rrbracket$ that have large Fourier coefficients. 

\noindent\textbf{Extension to Continuous Features.} \cite{Fourier:Heidari} extend Proposition \ref{prop:MohsenGram} to encompass a general Euclidean space. We use the generic construction of Fourier coefficients in the Euclidean space to extend our computations for feature spaces involving both categorical and continuous features. Rest of our computations follow naturally.

\noindent\textbf{Extension to Multi-class Classification.} We also deploy \texttt{AFA} for multi-class classification, where $\cY$ consists of multiple labels. In this setting, the concept of group fairness, i.e. $\GFair{h} \triangleq \max_{y \in \cY}|\prob_{\mathbf{x} \sim \mathcal{D}}[h(\mathbf{x}) = y |  x_A = 1] -  \prob_{\mathbf{x} \sim \mathcal{D}}[h(\mathbf{x}) = y |  x_A = - 1]|$, is called multicalibration~\citep{dwork2023pseudorandomness}. 
Here, we construct Fourier expansions of the model for each pair of labels. Then, we use Proposition~\ref{groupfourier} to compute the group fairness for each of the expansions, and finally, take the maximum to estimate multi-group fairness of $h$. Formal details are deferred to Appendix~\ref{sec:multiclass}. We experimentally evaluate both the extensions.
\begin{figure}[t!] 
    \centering
    \resizebox{0.8\columnwidth}{!}{
    \begin{tikzpicture}[level distance=1.5cm,
      level 1/.style={sibling distance=8cm},
      level 2/.style={sibling distance=4cm},
      level 3/.style={sibling distance=2cm},
      level 4/.style={sibling distance=1cm}]

      \LARGE
      \node {$\cB({\mathcal{X}})$}
        child {node {$\cB_0({\mathcal{X}})$}
          child {node {$\cB_{0,1}({\mathcal{X}})$}
            child {node {$\vdots$}
              child {node {$\cB_{\Upsilon_1}$}}
              child {node {$\cB_{\Upsilon_2}$}}
            }
            child {node {$\vdots$}
              child {node {$\ldots$}}
              child {node {$\ldots$}}
            }
          }
          child {node {$\cB_{0,\neg 1}({\mathcal{X}})$}
            child {node {$\vdots$}
              child {node {$\ldots$}}
              child {node {$\ldots$}}
            }
            child {node {$\vdots$}
              child {node {$\ldots$}}
              child {node {$\ldots$}}
            }
          }
        }
        child {node {$\cB_{\neg 0}({\mathcal{X}})$}
          child {node {$\cB_{\neg 0,1}({\mathcal{X}})$}
            child {node {$\vdots$}
              child {node {$\ldots$}}
              child {node {$\ldots$}}
            }
            child {node {$\vdots$}
              child {node {$\ldots$}}
              child {node {$\ldots$}}
            }
          }
          child {node {$\cB_{\neg 0,\neg 1}({\mathcal{X}})$}
            child {node {$\vdots$}
              child {node {$\ldots$}}
              child {node {$\ldots$}}
            }
            child {node {$\vdots$}
              child {node {$\cB_{\Upsilon_{k-1}}$}}
              child {node {$\cB_{\Upsilon_k}$}}
            }
          }
        };
    \end{tikzpicture}}
    \caption{\afa{} begins with the set of all Fourier coefficients, with weight $1$, which is above the threshold $\tau<1$. It proceeds by splitting the bucket and verifies at each level of the tree the weight of the node. If the weight is below the threshold, the algorithm halts. Otherwise, it continues to expand, yielding a set of (informative) trajectories $\Upsilon$, the subsets with large Fourier coefficients are $\{ \cB_{\Upsilon_1}(\cX), \cdots,\cB_{\Upsilon_k}(\cX)\}$.}\label{Fig:treeGL}
\end{figure}
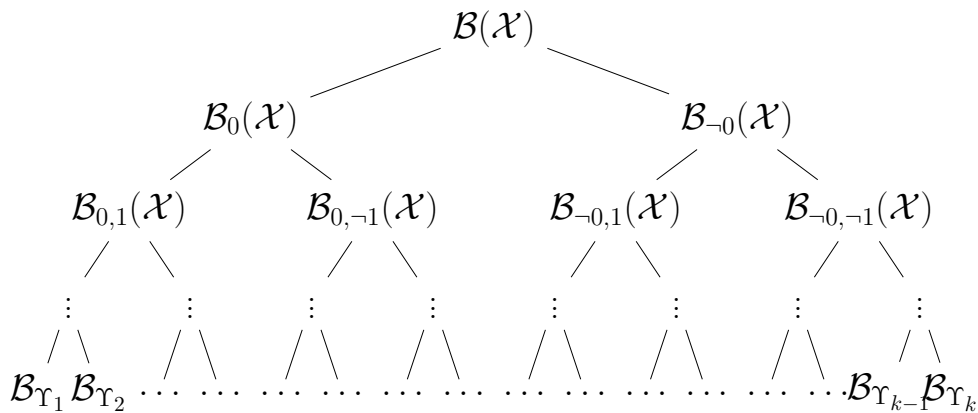

\vspace*{-1em}\section{Theoretical Analysis}\label{sec:MP}

\paragraph{Upper Bounds on Sample Complexity.}
\begin{theorem}[Upper bounds for Robustness and Individual Fairness]
\label{theo:52}
 \afa{} is a PAC-agnostic auditor for robustness and individual fairness with sample complexity $\cO \Big(\frac{\texttt{char}(L, \mu) ( 1 - 4 \texttt{char}(\Bar{L}, \mu))}{\epsilon} \sqrt{\log \frac{2}{\delta}}  \Big)$.
 Here, $\texttt{char}({L}, \mu_{\cdot}) \triangleq \sum\limits_{S \in L} \texttt{char}(S, \mu_{\cdot})$ 
 and $\texttt{char}(\Bar{L}, \mu_{\cdot}) \triangleq \underset{S \in \Bar{L}}{\sum} \texttt{char}(S, \mu_{\cdot})$.
\end{theorem}

\begin{theorem}[Upper bounds for Group Fairness]
\label{theo:53}
\afa{} yields an $(\epsilon,\delta)$-PAC estimate of $\GFair{h}$ if it has access to predictions of $\cO \left(\frac{1}{ \epsilon^2} \log \frac{4}{\delta}\right)$ input samples. 
\end{theorem}
We prove that \afa{} achieves an optimal rate of $\Tilde{\cO}(\frac{1}{\epsilon} \sqrt{\log \frac{1}{\delta}})$ for robustness and individual fairness and an $\Tilde{\cO}(\frac{1}{\epsilon^2} \log \frac{1}{\delta})$ rate for group fairness. 
Consequently, under the same number of samples, \afa{} exhibits a higher error rate for group fairness compared to robustness and individual fairness, as group fairness involves solving a quadratic equation while the others correspond to their respective influence functions. 
The proofs of these theorems are in Appendix~\ref{app:theory}.

\noindent\textbf{Rethinking Manipulation-proof.} 
\cite{Reconstruct&Audit:Yan} first propose manipulation-proof auditing that primarily revolves around fully reconstructing the model, and defines the manipulation-proof subclass using a version space. However, this approach may overlook numerous other models that, while having a significant probability mass in areas where they disagree with the black-box model, exhibit similar behavior to the black-box model w.r.t. the property. In contrast, we propose to capture all those functions by defining only the essential information required for auditing.
\begin{definition}[Fourier strategic manipulation-proof]\label{MP}
Let $h$ be a model that admits a Fourier expansion as in $h = \sum\nolimits_{S \subseteq [n]} \hat{h}(S) \psi_S$. We say that an auditor \(\mathcal{A}\) achieves optimal manipulation-proofness for estimating a (distributional) property \(\mu\) when \(\mathcal{A}\) is a PAC-agnostic auditor (Definition~\ref{pacverifier}) and outputs an exponential-size subclass of functions that satisfies $\forall h, h' \in \mathcal{M}, \prob\left( |\mu(h) - \mu(h')| \geq \epsilon \right) \leq \delta$.
\end{definition}

\begin{theorem}[Manipulation-proofness of \afa{}]
\label{thm:mp_afa}
\afa{} achieves optimal manipulation-proofness for estimating statistical parity with manipulation-proof subclass of size $2^{n-2}$.
\end{theorem}

\noindent\textbf{Lower Bounds without Manipulation-proofness.} 
In the following, we propose a lower bound for yielding a PAC estimate of the statistical parity with no manipulation-proof constraint. Additionally, we assume the auditing algorithm can sequentially query the black-box model with informative queries. The proof is in Appendix \ref{app:LB}.

\begin{theorem}[Lower bound without manipulation-proofness]\label{Th:lowerbounds}
    Let $\epsilon \in (0, 1)$, $\delta \in (0,1/2]$. We aim to obtain $(\epsilon, \delta)$-PAC estimate of SP of model $h \in \mathcal{H}$, where the hypothesis class $\cH$ has VC dimension $ d$. For any auditing algorithm $\cA$, there exists an adversarial distribution realizable by the model to audit such that with $\Tilde{\Omega}(\frac{\delta}{\epsilon^2} )$ samples, $\cA$ outputs an estimate $\hat{\mu}$ of $\GFair{h^*}$ with $\prob[|\hat{\mu} - \GFair{h^*}  |> \epsilon]> \delta$.
\end{theorem}
Our results extend the existing sample complexity results with model reconstruction~\citep{Reconstruct&Audit:Yan}, and also provide a reference of optimality for upper bounds. We highlight the gap from the upper bound established in Theorem~\ref{theo:53}, attributed to the lack of the manipulation proof.

\section{Empirical Performance Analysis}\label{sec:experiments}
In this section, we evaluate the performance of \afa{} in estimating multiple models' group fairness, robustness, and individual fairness. Below, we provide a detailed discussion of the experimental setup, objectives, and results.
\setlength{\textfloatsep}{4pt}
\begin{table}[t!]
\centering
\caption{Average estimation error for statistical parity across different ML models. `\textemdash' denotes when a method cannot scale to the model. The best method is in \textbf{bold}.}\label{tab:group_fairness} 
\resizebox{0.7\columnwidth}{!}{
\begin{tabular}{l|ccccccccc}
\toprule
Dataset & \multicolumn{3}{c}{COMPAS} & \multicolumn{3}{c}{Student} & \multicolumn{3}{c}{Drug}\\
\cmidrule(lr){2-4}\cmidrule(lr){5-7} \cmidrule(lr){8-10}     
Model                             & LR             & MLP            & RF             & LR             & MLP            & RF             & LR             & MLP            & RF             \\
\midrule
$\mu$CAL                          & 0.312          & \textemdash             & \textemdash              & \textemdash              &  \textemdash         & \textemdash              & \textemdash           & \textemdash            & \textemdash           \\
\texttt{Uniform} & 0.077          & 0.225          & 0.077          & 0.132          & 0.225          & 0.077          & 0.254          & 0.116          & 0.127          \\
AFA                               & \textbf{0.006} & \textbf{0.147} & \textbf{0.006} & \textbf{0.030} & \textbf{0.147} & \textbf{0.006} & \textbf{0.220} & \textbf{0.040} & \textbf{0.120}\\
\bottomrule
\end{tabular}}
\end{table}

\textbf{Experimental Setup.} 
We conduct experiments on COMPAS~\citep{angwin2016machine}, student performance (Student)~\citep{cortez2008using}, and drug consumption (Drug)~\citep{fehrman2019personality} datasets.  The datasets contain a mix of binary, categorical, and continuous features for binary and multi-class classification.  
We evaluate $\afa$ on three ML models: Logistic Regression (LR), Multi-layer Perceptron (MLP), and Random Forest (RF). The ground truth of group fairness, individual fairness, and robustness is computed using the entire dataset as in~\citep{Reconstruct&Audit:Yan}. 
For group fairness, we compare \afa{} with uniform sampling method, namely \texttt{Uniform}, and the active fairness auditing algorithms~\citep[Algorithm 3]{Reconstruct&Audit:Yan}, i.e. \texttt{CAL} and its variants $\mu\mathtt{CAL}$ and randomized $\mu\mathtt{CAL}$, which requires more information about the model class than black-box access. {We report the best variant of  \texttt{CAL} with the lowest error.} For robustness and individual fairness, we compare \afa{} with \texttt{Uniform}. Each experiment is run $10$ times and we report the averages. We refer to Appendix~\ref{appAA} for details.

\begin{figure}[t!]
    \centering
    \subfloat{\includegraphics[scale = 0.5]{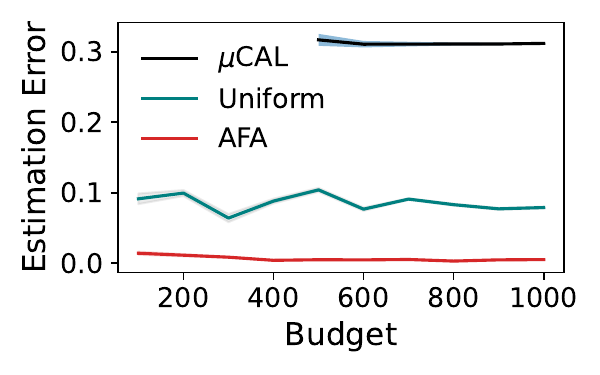}}
    \subfloat{\includegraphics[scale = 0.5]{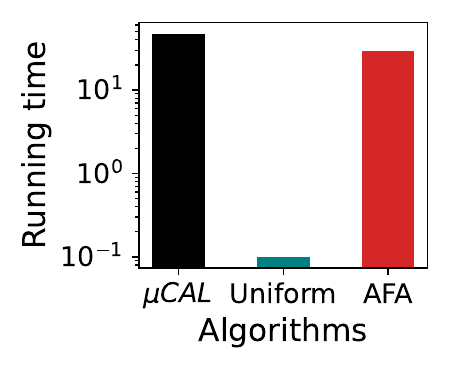}}
    \caption{Error (left) and running time (right) of different auditors in estimating statistical parity of COMPAS in LR.}\label{fig:group_fairness}
\end{figure}

\setlength{\textfloatsep}{4pt}
\begin{table}[t!]
  \centering
      \caption{Estimation error for robustness and individual fairness by $\mathtt{Uniform}$ and {\afa}. \textbf{Bold} case means lower error.} \label{tab:rho}
  \small{
    \begin{tabular}{crrrr}
    \toprule
     & \multicolumn{2}{c}{Robustness} & \multicolumn{2}{c}{Individual Fairness}\\
     \cmidrule(lr){2-3}\cmidrule(lr){4-5}
     $\rho$ &  \texttt{Uniform} &  \afa{}  & \texttt{Uniform} & \afa{} \\ 
   \midrule

$0.25$ &    $0.033$   &    $\mathbf{0.016}$     &    $0.036$    & $\mathbf{0.029}$ \\


$0.30$ &   $0.333$   &       $\mathbf{0.078}$  &  $0.309$     &  $\mathbf{0.047}$ \\


$0.35$ &   $0.299$    &    $\mathbf{0.139}$   &  $0.248$  &  $\mathbf{0.092}$ 
\\

    \bottomrule
    \end{tabular}}
\end{table}


Our empirical studies have the following \textbf{objectives}:\\
1. How accurate \afa{} is with respect to the baselines to audit robustness, individual fairness, and group fairness for different models and datasets?\\
2. How sample efficient and computationally efficient \afa{} is with baselines in auditing distributional properties?



\textbf{Accurate, Sample Efficient, and Fast Estimation of Group Fairness.}  
In Table~\ref{tab:group_fairness}, we demonstrate the estimation error of group fairness by different methods across datasets and models. {\afa{}} yields the lowest estimation error, hence a better method, than all baselines in all nine configurations of models and datasets. Among baselines, \texttt{CAL} cannot estimate group fairness beyond COMPAS on LR, due to the requirement of a finite version space, which is provided only for COMPAS on LR. \texttt{Uniform}, albeit simple to implement, invariably demonstrates erroneous estimate. \textit{Thus, {\afa} is the most accurate auditor for group fairness w.r.t. baselines.}


Figure~\ref{fig:group_fairness} (left) demonstrates the sample efficiency of different methods for statistical parity. {\afa} requires the lowest number of samples to reach almost zero estimation error. Thus, \textit{{\afa} is sample efficient than other methods}. Figure~\ref{fig:group_fairness} (right) demonstrates the corresponding runtimes, where {\afa{}} is the second fastest method after $\mathtt{Uniform}$ and faster than $\mathtt{CAL}$. \textit{Therefore, {\afa{}} yields a well balance between accuracy, sample efficiency, and running time among baselines.}

\textbf{Accurate Estimation of Robustness and Individual Fairness.} Table~\ref{tab:rho} demonstrates the estimation error for robustness and individual fairness achieved by {\afa} and \texttt{Uniform} with different $\rho$'s and 1000 samples from COMPAS dataset and LR model. {\afa} yields lower estimation error than \texttt{Uniform} across different models, and for higher values of $\rho$, the improvement due to {\afa} increases. Intuitively, \texttt{Uniform} samples IID from the space of input features, perturbs samples uniformly randomly, then queries the black-box model to obtain labels of perturbed samples to estimate properties. In contrast, $\afa{}$ queries samples recursively to cover the feature space and estimates large Fourier coefficients without perturbing the input features.
This also reflects the theoretical sample complexity results for \texttt{Uniform} and \afa, i.e. $O(1/\epsilon^2)$ and $O(1/\epsilon)$, respectively. \textit{Thus, {\afa} is more accurate than \texttt{Uniform} to estimate robustness and individual fairness.}

\vspace*{-.5em}\section{Conclusion and Future Work}

We propose {\afa}, a Fourier-based model-agnostic and black-box approach for universally auditing an ML model’s distributional
properties. We focus on three properties: robustness, individual fairness, and group fairness. We show that the significant Fourier coefficients of the black-box model yield a PAC approximation of all properties, establishing {\afa} as a universal auditor of ML. Empirically, {\afa} is more accurate, and sample efficient, while being competitive in running time than existing methods across datasets. In the future, we aim to extend \afa{} to estimate distributional properties other than the three studied in this paper.

\section{Acknowledgements}
\input{acks}

\bibliographystyle{apalike}
\bibliography{main}
\newpage
\appendix

\section{The cost of auditing with reconstruction: Proof of Proposition 2}\label{app:costreconstruct}
\begin{repproposition}{recost}
	If $\hat{h}$ is the reconstructed model from $h$, then
	\begin{align}\label{cost:rec}
		|\GFair{\hat{h}} - \GFair{h}| \leq \min \Bigg\lbrace 1, \frac{\prob_{\bx \sim \cD}[\hat{h}(\bx) \neq h(\bx)]}{\min(\prob_{\mathbf{x} \sim \mathcal{D}}[\bx_A = 1], \prob_{\mathbf{x} \sim \mathcal{D}}[\bx_A = -1])}\Bigg\rbrace.
	\end{align}
\end{repproposition}
\begin{proof}
\textbf{Step 1.} We begin the proof by lower bounding the probability of yielding different predictions by $h$ and $\hat{h}$.
\begin{align*}
    \underset{\mathbf{x} \sim \cD}{\prob}[\hat{h}(\mathbf{x}) \neq h(\mathbf{x})] & = \underset{\mathbf{x} \sim \cD}{\prob}[\hat{h}(\mathbf{x}) \neq h(\mathbf{x})|x_A= 0] \underset{\mathbf{x} \sim \cD}{\prob}[x_A=0] + \underset{\mathbf{x} \sim \cD}{\prob}[\hat{h}(\mathbf{x}) \neq h(\mathbf{x})|x_A= 1] \underset{\mathbf{x} \sim \cD}{\prob}[x_A=1] \\
    & \geq p \Bigl(\underset{\mathbf{x} \sim \cD}{\prob}[\hat{h}(\mathbf{x}) \neq h(\mathbf{x})|x_A= 0] + \underset{\mathbf{x} \sim \cD}{\prob}[\hat{h}(\mathbf{x}) \neq h(\mathbf{x})|x_A= 1]  \Bigl) \\
\end{align*}
The first equality is a consequence of the law of total probability. The last inequality holds as we define $p \triangleq \min\{\underset{\mathbf{x} \sim \cD}{\prob}[x_A=1], \underset{\mathbf{x} \sim \cD}{\prob}[x_A=0]\}$.

Since $p \neq 0$, we get
\begin{align}\label{partt0}
    \frac{1}{p}    \underset{\mathbf{x} \sim \cD}{\prob}[\hat{h}(\mathbf{x}) \neq h(\mathbf{x})] 
    \geq  \underset{\text{Term }1}{\underbrace{\underset{\mathbf{x} \sim \cD}{\prob}[\hat{h}(\mathbf{x}) 
    \neq h(\mathbf{x})|x_A= 0]}} + \underset{\text{Term }2}{\underbrace{\underset{\mathbf{x} \sim \cD}{\prob}[\hat{h}(\mathbf{x}) \neq h(\mathbf{x})|x_A= 1]}}.
\end{align}

\textbf{Step 2.} We observe that the Term 2 above can be rewritten as
\begin{align*}
    \underset{\mathbf{x} \sim \cD}{\prob}[\hat{h}(\mathbf{x}) = 1|x_A= 1] & = \underset{\mathbf{x} \sim \cD}{\prob}[\hat{h}(\mathbf{x}) = 1, h(\mathbf{x}) = - 1|x_A= 1] + \underset{\mathbf{x} \sim \cD}{\prob}[\hat{h}(\mathbf{x}) = 1, h(\mathbf{x}) = 1|x_A= 1] \\
    & \leq \underset{\mathbf{x} \sim \cD}{\prob}[\hat{h}(\mathbf{x}) \neq h(\mathbf{x}) |x_A= 1] + \underset{\mathbf{x} \sim \cD}{\prob}[ h(\mathbf{x}) = 1|x_A= 1]
\end{align*}
The last inequality is true due to the fact that $\hat{h}(\mathbf{x}) = 1, h(\mathbf{x}) = - 1$ is a sub-event of the event $h(x) \neq \hat{h}(x)$.

Now, by symmetry of $h$ and $\hat{h}$, we get
 \begin{align}
 \label{partt1}
     \Biggl| \underset{\mathbf{x} \sim \cD}{\prob}[\hat{h}(\mathbf{x}) = 1|x_A= 1] - \underset{\mathbf{x} \sim \cD}{\prob}[ h(\mathbf{x}) = 1|x_A= 1] \Biggl|  \leq \underset{\mathbf{x} \sim \cD}{\prob}[\hat{h}(\mathbf{x}) \neq h(\mathbf{x}) |x_A= 1]
 \end{align}

Similarly, working further with the Term 1 yields
\begin{align}\label{partt2}
     \Biggl| \underset{\mathbf{x} \sim \cD}{\prob}[\hat{h}(\mathbf{x}) = 1|x_A= 0] - \underset{\mathbf{x} \sim \cD}{\prob}[ h(\mathbf{x}) = 1|x_A= 0] \Biggl|  \leq \underset{\mathbf{x} \sim \cD}{\prob}[\hat{h}(\mathbf{x}) \neq h(\mathbf{x}) |x_A= 0]
 \end{align}

\textbf{Step 3.} Finally, using triangle inequality yields 
\begin{align*}
    | \mu(\hat{h}) - \mu(h) | & \leq \Biggl| \underset{\mathbf{x} \sim \cD}{\prob}[\hat{h}(\mathbf{x}) = 1|x_A= 0] - \underset{\mathbf{x} \sim \cD}{\prob}[ h(\mathbf{x}) = 1|x_A= 0] \Biggl| + 
    \Biggl| \underset{\mathbf{x} \sim \cD}{\prob}[\hat{h}(\mathbf{x}) = 1|x_A= 1] - \underset{\mathbf{x} \sim \cD}{\prob}[ h(\mathbf{x}) = 1|x_A= 1] \Biggl| \\
    & \leq \underset{\mathbf{x} \sim \cD}{\prob}[\hat{h}(\mathbf{x}) \neq h(\mathbf{x}) |x_A= 0]+ \underset{\mathbf{x} \sim \cD}{\prob}[\hat{h}(\mathbf{x}) \neq h(\mathbf{x}) |x_A= 1] \\
    & \leq \frac{1}{p}    \underset{\mathbf{x} \sim \cD}{\prob}[\hat{h}(\mathbf{x}) \neq h(\mathbf{x})]
\end{align*}
The second step comes from inequalities \eqref{partt1} and \eqref{partt2}, while the last one is due to inequality \eqref{partt0}.
\end{proof}

\section{Computing Model's Properties with Fourier Coefficients: Proofs of Section 3.2}\label{app:fourierpattern}

\subsection{Robustness and Individual Fairness}\label{app:fourierobust}
\begin{repproposition}{prop44}
	Let $\rho \in [-1,1]$. 
 
 The robustness of a binary classifier $h: \{-1,1\}^n \to \{-1,1\}$ under the $\Gamma_{\rho}$ flipping perturbation is equivalent to the $\rho$-flipping influence function, and thus, can be expressed as
	\begin{align*}
		\robust{h}  = \mathrm{Inf}_{\rho} (h) = \sum_{S \subseteq [n]} \rho^{|S|} \hat{h}(S)^2.
	\end{align*}
\end{repproposition}  

\begin{proof}\,

\textbf{Step 0: Robustness in terms of a composition of expectations over the perturbation $\Gamma_{\rho}$ and $\cD$.}
By the definition of robustness, we have:

\begin{align*}
    \mathrm{Inf}_{\rho}(h) &= \underset{\substack{\bx \sim \cD \\ \by \sim \Gamma_{\rho}(x) }}{\prob} [h(\bx) \neq h(\by)] \\
    &= \underset{\substack{\bx \sim \cD \\ \by \sim \Gamma_{\rho}(x) }}{\E} [h(\bx) h(\by)]
\end{align*}
Where the second equation comes from the fact that $h$ takes values in $\{-1,1\}$.

\textbf{Step 1: Robustness via operator approach.} We commence the proof by defining the robust operator $\mathcal{T}_{\rho}: \{-1,1\}^n \to \mathbb{R}$ as
$$\mathcal{T}_{\rho} h (\bx) \triangleq \underset{\by \sim N_{\rho}(\bx)}{\mathbb{E}}[h(\by)]\,.$$

Given the expression of the influence function in step 0, we have: 

\begin{align*}
\mathrm{Inf}_{\rho} (h) &= \underset{\substack{\mathbf{x} \sim \mathcal{D}\\ \by \sim N_{\rho}(\bx)}}{\mathbb{E}}[h(\mathbf{x}) h(\by)] \\
&= \underset{\mathbf{x} \sim \mathcal{D}}{\mathbb{E}}[h(\mathbf{x})  \underset{\by \sim N_{\rho}(\bx)} {\mathbb{E}} h(\by)] \\
&= \underset{\mathbf{x} \sim \mathcal{D}}{\mathbb{E}}[h(\mathbf{x}) \mathcal{T}_{\rho}h(\mathbf{x})] \\
& \triangleq \langle\ h,\mathcal{T}_{\rho}h \rangle_{\cD} \\
\end{align*}

The second equation comes from the linearity of the expectation, and the last step comes from the definition of the inner product that depends on the distribution $\cD$. 

Now, we expand both the model $h$ and the operator $\cT_{\rho} h$ in the Gram-Schmidt basis given by \cite{Fourier:Heidari}:

\begin{align*}
\mathrm{Inf}_{\rho} (h) & = \langle\ \sum_{i_1 = 1 }^{2^n} \hat{h}(S_{i_1}) \psi_{S_{i_1}}(\cdot), \sum_{i_2 = 1 }^{2^n} \hat{h}(S_{i_2}) \underset{\by \sim N_{\rho}(\cdot)}{\mathbb{E}}[ \psi_{S_{i_2}}(\by)] \rangle_{\cD} \\
&= \sum_{i_1 = 1 }^{2^n} \sum_{i_2 = 1 }^{2^n} \hat{h}(S_{i_1}) \hat{h}(S_{i_2}) \langle\ \psi_{S_{i_1}} , f^{\rho}_{S_{i_2}} \rangle_{\cD}
\end{align*}

Where, for all $\bx \in \{-1,1\}^n$ and for all $i \in \{1, \cdots, 2^n\}$, we used the following notation:  $f^{\rho}_{S_i}(x) \triangleq \underset{\by \sim N_{\rho}(x)}{\mathbb{E}}[ \psi_{S_i}(\by)]$.

\textbf{Step 2: Reduction of the robust operator to basis elements operators.} 

Let $\bx \in \cX$, 

\begin{align*}
    f^{\rho}_{S_i}(\bx) &= \underset{\by \sim N_{\rho}(\bx)}{\mathbb{E}}[ \psi_{S_i}(\by)] \\
                &= \underset{\by \sim N_{\rho}(\bx)}{\mathbb{E}}[ \chi_{S_i}(\by)] - \sum_{j=1}^{i-1} \alpha_{i,j} \underset{\by \sim N_{\rho}(\bx)}{\mathbb{E}}[ \psi_{S_j}(\by)] \\
                &= \underset{\by \sim N_{\rho}(\bx)}{\mathbb{E}}[ \prod_{k \in S_i} y_k] - \sum_{j=1}^{i-1} \alpha_{i,j} f^{\rho}_{S_j}(\bx) \\
                 &= \prod_{k \in S_i} \underset{\by \sim N_{\rho}(x)}{\mathbb{E}}[ y_k] - \sum_{j=1}^{i-1} \alpha_{i,j} f^{\rho}_{S_j}(\bx) \\
                 &= \prod_{k \in S_i}  \rho x_k - \sum_{j=1}^{i-1} \alpha_{i,j} f^{\rho}_{S_j}(\bx) \\
\end{align*}

The second inequality comes from replacing each basis element from the Gram-Schmidt orthogonalization process with its expression in terms of parity functions. In the fourth equation, the expectation over each component is computed by the perturbation process $\Gamma_{\rho}$, that is for each $k$ in $S_i$, $x_k$ is flipped with probability $\frac{1 - \rho}{2}$.

\begin{align*}
f^{\rho}_{S_i}(\bx) &= \rho^{|S_i|} \chi_{S_i}(\bx) - \sum_{j=1}^{i-1} \alpha_{i,j} f^{\rho}_{S_j}(\bx) \\
                &= \rho^{|S_i|} \psi_{S_i}(\bx) + \rho^{|S_i|} \sum_{j=1}^{i-1} \alpha_{i,j} \psi_{S_j}(\bx) - \sum_{j=1}^{i-1} \alpha_{i,j} f^{\rho}_{S_j}(\bx) \\
                &= \rho^{|S_i|} \psi_{S_i}(\bx) +  \sum_{j=1}^{i-1} \alpha_{i,j} (\rho^{|S_i|} \psi_{S_j}(\bx) - f^{\rho}_{S_j}(\bx))\,.
\end{align*}

Now, we compute the inner product left in step 1 to conclude the proof:

\begin{align*}
    \langle\ \psi_{S_i} , f^{\rho}_{S_k} \rangle_{\cD} &= \langle\ \psi_{S_i} , \rho^{|S_k|} \psi_{S_k} +  \sum_{j=1}^{k-1} \alpha_{k,j} (\rho^{|S_k|} \psi_{S_j} - f^{\rho}_{S_j}) \rangle_{\cD} \\
    &= \rho^{|S_k|} \delta_{i,k} +  \sum_{j=1}^{k-1} \delta_{j,k} \alpha_{k,j} (\rho^{|S_k|}  - \langle\ \psi_{S_i} , f^{\rho}_{S_j} \rangle_{\cD} ) \\
    \langle\ \psi_{S_i} , f^{\rho}_{S_k} \rangle_{\cD} & = \rho^{|S_k|} \delta_{i,k}
\end{align*}

Where the last step comes from the fact that $j < k$.

\textbf{Step 4: Conclusion. }\begin{align*}
    \mathrm{Inf}_{\rho} (h)          &= \sum_{i_1 = 1 }^{2^n} \sum_{i_2 = 1 }^{2^n} \hat{h}(S_{i_1}) \hat{h}(S_{i_2}) \langle\ \psi_{S_{i_1}} , f^{\rho}_{S_{i_2}} \rangle_{\cD} \\
     &= \sum_{i_1 = 1 }^{2^n} \sum_{i_2 = 1 }^{2^n} \hat{h}(S_{i_1}) \hat{h}(S_{i_2}) \rho^{|S_{i_1}|} \delta_{i_1,i_2} \\
     &= \sum_{i = 1 }^{2^n} \rho^{|S_i|}  \hat{h}(S_i)^2 \\
     &= \underset{S \subseteq [n]}{\sum}  \rho^{|S|}  \hat{h}(S)^2 
\end{align*}

We deduce the Fourier pattern in robustness property:

$$ \robust{h} =  \underset{S \subseteq [n]}{\sum}  \rho^{|S|}  \hat{h}(S)^2    $$
\end{proof}

The proof for individual fairness proceeds similarly by considering the operator $\mathcal{T}_{\rho}: \{-1,1\}^n \to \mathbb{R}$, defined as:
$$\mathcal{T}_{\rho,l} h (\mathbf{x}) = \mathbb{E}_{\mathbf{y} \sim N_{\rho,l}(\mathbf{x})}[h(\mathbf{y})]$$

\subsection{Group Fairness: Statistical Parity}
We first establish the relationship between group fairness and Fourier coefficients.
\begin{lemma}
If $\mathrm{Inf}_{A} (h)$ denotes the membership influence function for the sensitive attribute $A$ of the model $h$, we have the following result that relates the influence function to the model's $h$ Fourier coefficients:
    $$\mathrm{Inf}_{A} (h) = \underset{\substack{S \subseteq[n] \\ S \ni A}}{\sum} \hat{h}(S)^2$$
\end{lemma}

\begin{proof}
The membership influence function for the sensitive attribute $A$ is given by: 

    $$\mathrm{Inf}_{A} (h) = \underset{\substack{\bx \sim \mathcal{D}^{+}\\ \by \sim \mathcal{D}^{-}}}{\prob}[h(\bx) \neq h(\by)]$$

This function is closely related to the Laplacian of the target model in the direction of the sensitive attribute $A$, defined as:

$$L_A h(\mathbf{x},\mathbf{y}) := \frac{h(\bx) - h(\by)}{2}, \forall (\bx,\by) \in (\cX^+, \cX^{-})$$

Since $h$ takes values in $\{-1,1\}$, one can see that $ |L_A h(\bx,\by)|^2 =  \mathds{1}_{\{ h(\bx) \neq h(\by) \}}$.

By taking the expectation over the left and right part: 

$$\lVert L_A h \rVert_{\mathcal{D}^{+},\mathcal{D}^{-}}^2 = \underset{\substack{\mathbf{x} \sim \mathcal{D}^{+} \\ \mathbf{y} \sim \mathcal{D}^{-}}}{\mathbb{E}}[ L_A h(\mathbf{x},\mathbf{y})^2] = \mathrm{Inf}_A(h)$$

\begin{equation*} 
\begin{split}
\forall (\mathbf{x}, \mathbf{y}) \in \mathcal{X}^{+} \times \mathcal{X}^{-} : L_A h(\mathbf{x},\mathbf{y})  & = \frac{1}{2} \underset{S \subseteq [n]}{\sum} \hat{h}(S) \psi_S(\mathbf{x}) - \frac{1}{2} \underset{S \subseteq[n] }{\sum} \hat{h}(S) \psi_S(\mathbf{y})  \\
& = \frac{1}{2} \underset{\substack{S \subseteq[n] \\ S \ni A}}{\sum} \hat{h}(S) \psi_S(\mathbf{x}) + \frac{1}{2} \underset{\substack{S \subseteq [n] \\ S \not\ni A}}{\sum} \hat{h}(S) \psi_S(\mathbf{x}) \\ &~~~ - \frac{1}{2} \underset{\substack{S \subseteq[n] \\ S \ni A}}{\sum} \hat{h}(S) \psi_S(\mathbf{y})  - \frac{1}{2} \underset{\substack{S \subseteq [n] \\ S \not\ni A}}{\sum} \hat{h}(S) \psi_S(\mathbf{y}) \\
\forall (\bx, \by) \in \cX^{+} \times \cX^{-}: L_Ah(\mathbf{x},, \mathbf{y}) & = \frac{1}{2} \underset{\substack{S \subseteq [n] \\ S \ni A}}{\sum} \hat{h}(S) \psi_S(\mathbf{x}) -  \frac{1}{2} \underset{\substack{S \subseteq[n] \\ S \ni A}}{\sum} \hat{h}(S) \psi_S(\mathbf{y}) 
\end{split}
\end{equation*}

By Parseval identity, $\lVert L_A h \rVert_{\mathcal{D}^{+},\mathcal{D}^{-}}^2 = \underset{\substack{S \subseteq[n] \\ S \ni A}}{\sum} \hat{h}(S)^2$.

Hence, 

$$\mathrm{Inf}_A(h) = \lVert L_A h \rVert_{\mathcal{D}^{+},\mathcal{D}^{-}}^2 = \underset{\substack{S \subseteq[n] \\ S \ni A}}{\sum} \hat{h}(S)^2$$
\end{proof}

\begin{repproposition}{groupfourier}
Statistical parity of $h$ w.r.t a sensitive attribute $A$ and distribution $\cD$ is the root of the second order polynomial \begin{align}\label{eq:polynomial_gf}
P_{\hat{h}}(X) \triangleq \alpha (1- \alpha) X^2 - \hat{h}(\emptyset)(1-2 \alpha) X - {\sum_{\substack{S \subseteq[n], S \ni A}}} \hat{h}(S)^2 - \frac{(  1 - \hat{h}^2(\emptyset)) }{2}\,,
\end{align} 
where {$\alpha = \underset{\mathbf{x} \sim \mathcal{D}}{\mathbb{P}}[x_A = 1]$ and $\hat{h}(\emptyset)$ is the coefficient of the empty set.}
\end{repproposition}

\begin{proof}
	We use the following notation in the proof:
	\begin{align*}
		p &= \underset{\bx \sim \cD}{\prob}[h(\bx) = 1]\\
		\alpha &= \underset{\bx \sim \cD}{\prob}[\cX^{+}] \quad \text{(probability of belonging to the first sensitive group)}\\
		\GF^{+}(h) &= \underset{\bx \sim \mathcal{D}}{\mathbb{P}}[h(\bx) = 1| x_A = 1] \\
		\GF^{-}(h) &= \underset{\bx \sim \mathcal{D}}{\mathbb{P}}[h(\bx) = 1| x_A = -1] \\
	\end{align*}
	
	We have, 
	
	\begin{align}\label{eq:plusminusGF1}
		\GFair{h} = \GF^{+}(h) - \GF^{-}(h)
	\end{align}
	
	By the law of total probability, we also have:
	
	\begin{align}\label{eq:plusminusGF2}
		p = \alpha \GF^{+}(h) + (1- \alpha) \GF^{-}(h)
	\end{align}
	
	We first express the membership influence function in terms of the statistical parity:
	
	\begin{align*}
		\mathrm{Inf}_A(h) & = \underset{\bx, \bx'\sim \cD}{\prob}[h(\bx) \neq h(\bx')|x_A = 1, x_A' = -1]\\
		& = \underset{\bx, \bx' \sim \mathcal{D}}{\mathbb{P}}[h(\bx) = 1, h(\bx') = 0| x_A = 1, x_A' = -1 ] +  \underset{\bx, \bx' \sim \mathcal{D} }{\prob}[h(\bx) = -1, h(\bx') = 1|x_A = 1, x_A' = -1]\\
		&= \GF^{+}(h) ( 1 - \GF^{-}(h)) + \GF^{-}(h) ( 1 - \GF^{+}(h))\\
		&= \GF^{+}(h) + \GF^{-}(h) -2 \GF^{+}(h) \GF^{-}(h)
	\end{align*}
	
	Hence, we have:
	
	$$ \GF^{+}(h) + \GF^{-}(h) -2 \GF^{+}(h) \GF^{-}(h) - \mathrm{Inf}_A(h) = 0$$

	From equation \ref{eq:plusminusGF1}, and equation~\ref{eq:plusminusGF2}, we have:
	
	$$\begin{cases}
		\GF^{+}(h) &= p + (1 - \alpha) \GFair{h}\\
		\GF^{-}(h) &=  p - \alpha \GFair{h}
	\end{cases}$$
	
	The expression becomes:
	
	\begin{align*}
		2 \alpha (1- \alpha) \GFair{h}^2 + (1 - 2 p )(1-2 \alpha) \GFair{h} - \mathrm{Inf}_A(h) + 2p(1-p)    = 0
	\end{align*}
	
	The Fourier coefficient of the empty set is given by:
	
	\begin{align*}
		\hat{h}(\emptyset) & = \underset{\mathbf{x} \sim \mathcal{D}}{\mathbb{E}}[h(\mathbf{x})] \\
		& = \underset{\mathbf{x} \sim \mathcal{D}}{\mathbb{E}}[2 \mathds{1}_{\{ h(\mathbf{x}) =1 \}} -1]\\
		\hat{h}(\emptyset) & = 2 \underset{\mathbf{x} \sim \mathcal{D}}{\mathbb{P}}[h(\mathbf{x}) = 1] -1
	\end{align*}
	
	Since $p = \underset{\mathbf{x} \sim \mathcal{D}}{\mathbb{P}}[h(\mathbf{x}) = 1]$, we get the desired result.

\end{proof}

\begin{corollary}
	If $\mathcal{D}$ is the uniform distribution, statistical parity is exactly the Fourier coefficient of the sensitive attribute, i.e.	
	$$\GFair{h} = \hat{h}(\{ A \} )$$
\end{corollary}
\begin{proof}
\begin{equation*} 
\begin{split}
\GFair{h} & = |\underset{ \mathbf{x} \sim \mathcal{D}}{\mathbb{P}}[h( x) = y |  x \in A^{+}] - \underset{ \mathbf{x} \sim \mathcal{D}}{\mathbb{P}}[h( x) = y |  x \in A^{-}]|  \\ 
& = |\underset{ \mathbf{x} \sim \mathcal{D}}{\mathbb{P}}[h( x) = y |  x \in A^{+}] - \frac{1}{2} -  \underset{ \mathbf{x} \sim \mathcal{D}}{\mathbb{P}}[h( x) = y |  x \in A^{-}] + \frac{1}{2}| \\
& = |\frac{1}{2}\underset{ \mathbf{x} \sim \mathcal{D}}{\mathbb{E}}[2 \mathds{1}_{\{ h(\mathbf{x}) =1 \}} -1 |  x \in A^{+}] -  \frac{1}{2}\underset{ \mathbf{x} \sim \mathcal{D}}{\mathbb{E}}[2 \mathds{1}_{\{ h(\mathbf{x}) =1 \}} -1 |  x \in A^{-}] | \\
& = |\frac{1}{2}\underset{ \mathbf{x} \sim \mathcal{D}}{\mathbb{E}}[h(\mathbf{x})|  x \in A^{+}] - \frac{1}{2}\underset{ \mathbf{x} \sim \mathcal{D}}{\mathbb{E}}[h(\mathbf{x}) |  x \in A^{-}] | \\
& = |\frac{1}{2}\underset{ \mathbf{x} \sim \mathcal{D}}{\mathbb{E}}[h(\mathbf{x})\psi_A(x)|  x \in A^{+}] - \frac{1}{2}\underset{ \mathbf{x} \sim \mathcal{D}}{\mathbb{E}}[h(\mathbf{x}) \psi_A(x) |  x \in A^{-}] | \\
\end{split}
\end{equation*}
\end{proof}

\section{Theoretical Analysis: Proofs of Section 4}\label{app:theory}

%
%
%
\subsection{Upper Bounds on Sample Complexity of \afa}
\begin{claim}\label{claim:finiteintersect}
	Let $\{A_i\}_{i \in \cI}$ a finite set of events indexed by $\cI$. Then,
	$$\prob \Big[\underset{i\in \cI}{\bigcap} A_i\Big] \geq \sum_{i \in\cI} \prob\Big[A_i\Big] -  |\cI| +1$$
\end{claim}
The proof is a consequence of the union bound.

\begin{lemma}[Two-Sample Hoeffding's Inequality]\label{lemma:2Shoeffding}
	If $X_1,...X_{m_1}$, $X'_1,...X'_{m_2}$ are iid random variables taking values in $[-1,1]$ generating by the distribution $\cD$, such that
	$$\mu = \E[X^2], \text{ and } \hat{\mu} = \frac{1}{m_1 m_2}\sum_{i=1}^{m_1} \sum_{j=1}^{m_2} X_i X'_j, $$ 
	then
	$$\prob[|\hat{\mu} - \mu| \leq 4 \epsilon] \geq 1 -2 \exp \Big\{- \frac{m_1 m_2 \epsilon^2}{8} \Big\}$$
\end{lemma}
The proof is obtained by employing one sample Hoeffding inequality to the random variable $Z_{i,j} =X_i X'_j$.

\begin{reptheorem}{theo:52}[Upper bounds for Robustness and Individual Fairness]
	Given $\epsilon \in (0,1)$ and $\delta \in (0,1]$, \afa{} is a PAC-agnostic auditor for robustness and individual fairness with sample complexity $$\cO \Big(\frac{\texttt{char}(L, \mu) ( 1 - 4 \texttt{char}(\Bar{L}, \mu))}{\epsilon} \sqrt{\log \frac{2}{\delta}}  \Big).$$
	Here, $\texttt{char}({L}, \mu_{\cdot}) \triangleq \sum\limits_{S \in L} \texttt{char}(S, \mu_{\cdot})$ 
	and $\texttt{char}(\Bar{L}, \mu_{\cdot}) \triangleq \underset{S \in \Bar{L}}{\sum} \texttt{char}(S, \mu_{\cdot})$.
\end{reptheorem}
\begin{proof}\,
	
\textbf{Step 0.}  Let us define $\tau^2 \triangleq 4 \epsilon.$
    
Let $x_1, \cdots, x_{m_1}, x'_1, \cdots, x'_{m_2}$ are sampled i.i.d. from $\cD$, where $m = m_1 + m_2$ denotes the total number of samples, and $m_1$ and $m_2$ to be the number of samples with $x_a=1$ and $x_A=-1$, respectively.

Let $L$ denote the list of subsets exhibiting Fourier coefficients larger than $\tau$. 
 
\textbf{Step 1.}  By definitions of $\texttt{char}(S, \mu_{\cdot})$, and the results of Proposition~\ref{prop44} and~\ref{prop47}, we unifiedly express both the `true' properties of $h$ as
\begin{align}
	\mu(h) &=  \sum_{S \subseteq \Iintv{1,n}} \texttt{char}(S, \mu) \hat{h}(S)^2 \notag\\
	&= \sum_{S \subseteq \Iintv{1,n}} \texttt{char}(S, \mu) \E_{\bx, \by \sim \cD}[h(\bx) h(\by) \psi_S(\bx) \psi_S(\by)]\,. \label{eq:true_rob_if}
\end{align}

Now, for any $S \in L$, we define an unbiased estimator of the squared Fourier coefficients as
\begin{align}\label{eq:estimate_square_fourier}
	\hat{h}_{\afa{}}(S)^2 \triangleq \frac{1}{m_1 m_2}\sum_{i=1}^{m_1} \sum_{j=1}^{m_2} h(x_i) h(x'_j) \psi_S(x_i) \psi_S(x_j)\,.
\end{align}    
 
Hence, the estimators of these properties, i.e. robustness and individual fairness, takes the form
\begin{align}\label{eq:estim_rob_if}
	\hat{\mu}_{\afa{}} \triangleq \frac{1}{m_1 m_2} \sum_{S\in L}\sum_{i=1}^{m_1}\sum_{j=1}^{m_2}  \texttt{char}(S, \mu) h(x_i) h(x'_j) \psi_S(x_i) \psi_S(x_j)\,.
\end{align}

\textbf{Step 2.} Using Equation~\eqref{eq:true_rob_if} and~\eqref{eq:estim_rob_if}, we express the estimation error as
\begin{align}
        |\mu(h) - \hat{\mu}_{\afa{}}| &= \Biggl|\sum_{S \subseteq \Iintv{1,n}} \texttt{char}(S, \mu) \hat{h}(S)^2 -   \sum_{S \in L} \texttt{char}(S, \mu) \hat{h}_{\afa{}}(S)^2 \Biggl|\notag\\
        &= \Biggl|\sum_{S \in L} \texttt{char}(S, \mu) \hat{h}(S)^2 + \sum_{S \not \in L} \texttt{char}(S, \mu) \hat{h}(S)^2 -   \sum_{S \in L} \texttt{char}(S, \mu) \hat{h}_{\afa{}}(S)^2 \Biggl|\notag\\
        & \leq \sum_{S \not \in L} \texttt{char}(S, \mu) \hat{h}(S)^2 +   \sum_{S \in L} \texttt{char}(S, \mu) \Big| \hat{h}(S)^2 - \hat{h}_{\afa{}}(S)^2 \Big|\notag \\
        & \leq \tau^2 \sum_{S \not \in L} \texttt{char}(S, \mu)  +   \sum_{S \in L} \texttt{char}(S, \mu) \Big| \hat{h}(S)^2 - \hat{h}_{\afa{}}(S)^2 \Big|\,. \label{eq:deterministic_bound_estim_rob_if}
\end{align}

The penultimate inequality is due to the fact that $|x+y| \leq |x|+|y|$ for all $x,y \in \R$.    The last inequality is by the definition of $L$, i.e. $\forall S \subseteq \Iintv{1,n}: S \not \in L$ implies that $|\hat{h}(S)| \leq \tau$. \afa{} gets access to this list of subsets $L$ due to the Goldreich-Levin algorithm.


\textbf{Step 3.}   Now, we leverage Equation~\eqref{eq:deterministic_bound_estim_rob_if}, to derive an PAC estimation bound for robustness and individual fairness. Specifically,
\begin{equation}\label{eq:charmu}
        \prob \Bigl[|\mu(h) - \hat{\mu}_{\afa{}}| \geq \epsilon \Bigl] \leq  \prob \Bigl[ \sum_{S \in L} \texttt{char}(S, \mu) | \hat{h}(S)^2 - \hat{h}_{\afa{}}(S)^2 | \geq \epsilon - \tau^2 \texttt{char}(\Bar{L}, \mu)   \Bigl]      
    \end{equation}
Here, we denote by $\texttt{char}(L, \mu)$ the sum $ \sum_{S \in L} \texttt{char}(S, \mu)$ and $\texttt{char}(\Bar{L}, \mu)$ the sum $\sum_{S \not \in L} \texttt{char}(S, \mu)$.
 
\textbf{Step 4.}  Now, by consecutively applying Claim~\ref{claim:finiteintersect} and Lemma~\ref{lemma:2Shoeffding}, we get an upper bound on the estimation error of the squared Fourier coefficients in $L$.

\begin{align*}
\prob\Bigl[ \underset{S \in L}{\bigcap}\Big\{\Big| \hat{h}(S)^2 - \hat{h}_{\afa{}}(S)^2 \Big|  \leq 4 \epsilon  \Big\}  \Bigl] &\geq \sum_{S \in L} \prob\Bigl[ \Big| \hat{h}(S)^2 - \hat{h}_{\afa{}}(S)^2 \Big| \leq 4 \epsilon   \Bigl] - |L| +1 \\
& \geq |L| - 2 |L| \exp\Bigl\{ - \frac{m_1 m_2 \epsilon^2}{8} \Bigl\} - |L| + 1 \\
& \geq 1 - 2 |L| \exp\Bigl\{ - \frac{m_1 m_2 \epsilon^2}{8} \Bigl\}  \\
\end{align*}

This result naturally yields a bound on $\sum_{S \in L}  \texttt{char}(S, \mu) \Big| \hat{h}(S)^2 - \hat{h}_{\afa{}}(S)^2 \Big|$.
    \begin{align*}
        \prob\Bigl[ \sum_{S \in L}  \texttt{char}(S, \mu) \Big| \hat{h}(S)^2 - \hat{h}_{\afa{}}(S)^2 \Big| \geq 4 \texttt{char}(L, \mu) \epsilon \Bigl] 
        & \leq \prob\Bigl[ \underset{S \in L}{\bigcup}\Big\{\Big| \hat{h}(S)^2 - \hat{h}_{\afa{}}(S)^2 \Big| \geq 4 \epsilon  \Big\}  \Bigl]\\
        & \leq 2|L| \exp\Bigl\{ - \frac{m_1 m_2 \epsilon^2}{8} \Bigl\} 
    \end{align*}
The last inequality is due to the union bound.

\textbf{Step 5.} Finally, using the fact that $4\epsilon = \tau^2$ and properly substituting to ensure $4 \texttt{char}(L, \mu) \epsilon \geq \epsilon - \tau^2 \texttt{char}(\Bar{L}, \mu)$, we get
\begin{align*}
    \prob\Bigl[ \sum_{S \in L}  \texttt{char}(S, \mu) \Big| \hat{h}(S)^2 - \hat{h}_{\afa{}}(S)^2 \Big| \geq  \epsilon - \tau^2 \texttt{char}(\Bar{L}, \mu) \Bigl] \leq 2 |L| \exp\Bigl\{ - \frac{m_1 m_2 \epsilon^2}{128 \texttt{char}(L, \mu) ^2 ( 1 - 4 \texttt{char}(\Bar{L}, \mu))^2} \Bigl\}
\end{align*}

Hence, by Equation~\eqref{eq:charmu}, 
$$\prob \Bigl[|\mu(h) - \hat{\mu}_{\afa{}}| \geq \epsilon \Bigl] \leq 2 |L| \exp\Bigl\{ - \frac{m_1 m_2 \epsilon^2}{128 \texttt{char}(L, \mu) ^2 ( 1 - 4 \texttt{char}(\Bar{L}, \mu))^2} \Bigl\}$$

By the definition of the sample complexity, the probability in the RHS has to be less than a given $\delta$. Thus, 
\begin{align*}
m_1 m_2 &\geq \frac{128 \texttt{char}(L, \mu) ^2 ( 1 - 4 \texttt{char}(\Bar{L}, \mu))^2}{\epsilon^2} \log \frac{2 |L|}{\delta}\,.
\end{align*}

Since $L\geq1$ and $m= m_1 +m_2 \geq 2 \sqrt{m_1 m_2}$, we conclude
$$m\geq \frac{8 \sqrt{2} \texttt{char}(L, \mu) ( 1 - 4 \texttt{char}(\Bar{L}, \mu))}{\epsilon} \sqrt{\log \frac{2}{\delta}} $$
\end{proof}
\clearpage
\begin{reptheorem}{theo:53}[Upper bounds for Group Fairness]
	Given $\epsilon \in (0,1)$ and $\delta \in (0,1]$, \afa{} yields an $(\epsilon,\delta)$-PAC estimate of $\GFair{h}$ if it has access to predictions of $$\cO \left(\frac{1}{ \epsilon^2} \log \frac{4}{\delta}\right)$$ input samples. 
\end{reptheorem}

\begin{proof}
	
\textbf{Step 1.} First, we aim to express the group fairness as a root of the second-order polynomial in Proposition~\ref{groupfourier}, and thus, to check when this approach is valid.

We observe that the discriminant of this second order polynomial is
\begin{align*}
     \Delta &= (2p+1)^2  (2 \alpha -1)^2 + 8 \alpha (1 - \alpha) \mathrm{Inf}_A -1 \\
     &= 4 \alpha^2 +4 p^2 -4\alpha - 4p+1 + 8 \alpha (1 - \alpha) \mathrm{Inf}_A \\
     &= 4 \alpha^2 +4 p^2 -4\alpha - 4p+1 + 8 \alpha (1 - \alpha) \sum_{S \subseteq\Iintv{1,n}} \hat{h}^2(S)\\
     &= 4 \alpha^2 +4 p^2 -4\alpha - 4p+1 + 8 \alpha (1 - \alpha) \sum_{S \in L} \hat{h}^2(S) + 8 \alpha (1 - \alpha) \sum_{S \not \in L} \hat{h}^2(S)\\
      & \geq 4 \alpha^2 +4 p^2 -4\alpha - 4p+1 + 8 \alpha (1 - \alpha) \sum_{S \in L} \hat{h}^2(S)\\
      & \geq 4 \alpha^2 +4 p^2 -4\alpha - 4p+1 + 8 |L| \tau^2 \alpha (1 - \alpha) \\
       & \geq 4 \alpha^2 +4 p^2 -4\alpha - 4p+1 + 32 \epsilon \alpha (1 - \alpha) \\
       & = 4 (1- 8 \epsilon) (\alpha - \frac{1}{2})^2 +4  (p - \frac{1}{2})^2 - (1- 8 \epsilon) 
\end{align*}

For $\epsilon > \frac{1}{8}$, $\Delta$ is positive. Thus, $\GFair{h}$, i.e. the zero of a second-order polynomial, can be expressed as
$$\GFair{h}=  \frac{-( 1 - 2 \alpha)(1 - 2p)  + \Biggl( 4 \alpha^2 +4 p^2 -4\alpha - 4p+1 + 8 \alpha (1 - \alpha) \mathrm{Inf}_A \Biggl)^{0.5} }{4 \alpha (1 - \alpha)}$$
Here, $p = \frac{1 + \hat{h}(\emptyset)}{2}$ and $ \mathrm{Inf}_A = \mathrm{Inf}_A(h) = \underset{\substack{ S \subseteq \Iintv{1,n},~S \ni A}}{\sum}  \hat{h}(S)^2$.

\textbf{Step 2.} We consider the following estimator yielded by \afa\footnote{Note that this estimator is independent of $\alpha$ or $p$, unlike the restrictive assumptions required in existing works~\citep{Reconstruct&Audit:Yan}.
}.

$$\hat{\mu}_{\mathrm{GFair}}(h) =  \frac{-( 1 - 2 \alpha)(1 - 2\hat{p})  + \Biggl(4 \alpha^2 +4 \hat{p}^2 -4\alpha - 4\hat{p}+1 + 8 \alpha (1 - \alpha) \widehat{\mathrm{Inf}_A} \Biggl)^{0.5} }{4 \alpha (1 - \alpha)}\,,$$

where
\begin{align*}
   \hat{p} = \frac{1 + \hat{h}_{\afa{}}(\emptyset)}{2}, \quad&\text{ and }\quad  \widehat{\mathrm{Inf}_A} = \underset{\substack{ S \in L \\ S \ni A}}{\sum}  \hat{h}_{\afa{}}(S)^2\,.
\end{align*}

To simplify notations, we denote: 

\begin{align}
    \Delta &= 4 \alpha^2 + 4 p^2 - 4 \alpha - 4 p + 8 \alpha (1 - \alpha) \mathrm{Inf}_A +1   \\
    \hat{\Delta} &=  4 \hat{\alpha}^2 + 4 \hat{p}^2 - 4 \alpha - 4 p + 8 \alpha (1 - \alpha) \widehat{\mathrm{Inf}_A} +1 
\end{align}

\textbf{Step 3.} We have, 
\begin{align*}
     \prob\biggl[|\widehat{\GF} - \GFair{h}| &\leq \epsilon  \biggl] \geq \prob\biggl[|\hat{p} - p| \leq \frac{2\alpha(1-\alpha)\epsilon}{|1-2\alpha|} \biggl] + \prob\biggl[|\hat{\Delta} - \Delta| \leq 2\alpha(1-\alpha)\epsilon \biggl] -1
\end{align*}

On the other hand,

\begin{align*}
    \prob\biggl[|\hat{\Delta} - \Delta| &\leq \epsilon  \biggl] \geq \prob\biggl[|\hat{p}^2 - p^2| \leq \frac{\epsilon}{12}  \biggl] + \prob\biggl[|\hat{p} - p| \leq \frac{\epsilon}{12} \biggl] + \prob\biggl[|\widehat{\mathrm{Inf}_A} - \mathrm{Inf}_A| \leq \frac{\epsilon}{24 \alpha (1 - \alpha)} \biggl] 
\end{align*}

Similar to the previous proof and we apply using Two-sample Hoeffding on the first and third term above, while we use the classical Hoeffding for the second term.
Together they yield a sample complexity upper bound of $\cO \Biggl(\max \biggl\{ \frac{1}{ \epsilon^2} \log \frac{4}{\delta}, \frac{1}{\epsilon} \sqrt{\log \frac{2}{\delta}} \biggl\} \Biggl)$, which is $\cO(\frac{1}{ \epsilon^2} \log \frac{4}{\delta})$ for $\epsilon \in (0,1)$ and $\delta \in (0,1]$.
\end{proof}

\subsection{Manipulation-proofness of \afa}

\begin{reptheorem}{thm:mp_afa}[Manipulation-proofness of \afa{}]
	\afa{} achieves optimal manipulation-proofness for estimating statistical parity with manipulation-proof subclass of size $2^{n-2}$.
\end{reptheorem}

\begin{proof}
	We are interested in hypotheses $h$ for which $\GFair{h} = \GFair{h^*}$.
	
	Let $h^*$ denote the model under audit and let $h$ be any model that admits Fourier decomposition, we have:
	
	\begin{align*}
		h &= \underset{S \subseteq [n]}{\sum} \hat{h}(S) \psi_S \\
		&=\underset{\substack{S \subseteq [n] \\ S \neq \emptyset}}{\sum} \hat{h}(S) \psi_S + \hat{h}(\emptyset) \psi_{\emptyset}\\
		&=\underset{\substack{S \subseteq [n] \\ S \neq \emptyset, S \ni A}}{\sum} \hat{h}(S) \psi_S +\underset{\substack{S \subseteq [n] \\ S \neq \emptyset,  S \not \ni A}}{\sum} \hat{h}(S) \psi_S+ \hat{h}(\emptyset) \psi_{\emptyset}\\
	\end{align*}
	
	On the other hand,
	
	$$\forall S : S \ni A, \hat{h}(S) = h^*(S), \hat{h}(\emptyset) = h^*(\emptyset) \implies \GFair{h} = \GFair{h^*}$$
	
	Where the last line comes from the dependence of statistical parity on the Fourier coefficients of the empty set and any subset that contains the protected feature (e.g, Formula~\ref{groupfourier}).
	
	Hence, the manipulation proof subclass is: $\Biggl\{h: \underset{S \subseteq [n]}{\sum} \hat{h}(S) \psi_S: \forall S \subseteq [n]: (S= \emptyset) \vee (S \ni A) \implies \hat{h}(S) = \hat{h^*}(S) \Biggl\}$, which has a size of $2^{n-2}$.
\end{proof}

\subsection{Lower Bound on Sample Complexity without Manipulation-proofness}\label{app:LB}
\begin{reptheorem}{Th:lowerbounds}[Lower bound without manipulation-proofness]
	Let $\epsilon \in (0, 1)$, $\delta \in (0,1/2]$. We aim to obtain $(\epsilon, \delta)$-PAC estimate of SP of model $h \in \mathcal{H}$, where the hypothesis class $\cH$ has VC dimension $ d$. For any auditing algorithm $\cA$, there exists an adversarial distribution realizable by the model to audit such that with $\Tilde{\Omega}(\frac{\delta}{\epsilon^2} )$ samples, $\cA$ outputs an estimate $\hat{\mu}$ of $\GFair{h^*}$ with $\prob[|\hat{\mu} - \GFair{h^*}  |> \epsilon]> \delta$.
\end{reptheorem}

\begin{proof}
Let $\mathcal{H}$ be a hypothesis class of VC dimension $\text{VC}(\cH)$, we start with case $\text{VC}(\cH) \in 2 \N$.

Let $\cZ = \{\zeta_1, \dots,\zeta_d, \zeta_{d+1}, \cdots, \zeta_{2d}\} \subseteq \mathcal{X}$ a subspace shattered by $\mathcal{H}$, let $N$ be our querying budget.

\noindent\textbf{Step 1: Construction of adversarial distribution.} Let $\cZ^{+} = \{\zeta_1, \dots,\zeta_d\}$ and $\cZ^{-} = \{\zeta_{d+1}, \dots,\zeta_{2d}\}$.

We define the adversarial distribution as the distribution satisfying:

    \[ \cD = \begin{cases} 
          x|\cX^{+} & \sim \mathcal{U} \{\cZ^{+}\}\\
          x|\cX^{-} & \sim \mathcal{U} \{\cZ^{-}\}
       \end{cases}
    \]
For any $i \in \Iintv{1,2d}$ and given the iid assumption, any $z \sim \cZ^{+}$ will be denoted $z^{+}$ and similarly any $z \sim \cZ^{-}$ will be denoted $z^{-}$.
 
Consider hypotheses $H_0$ and $H_1$ that chooses $h^*$ randomly from $\{0,1\}^{\cZ}$:

\begin{itemize}
\item $H_0$: picks $h^*$ such that for all $i \in \Iintv{1,d}$ independently: 
\begin{equation}
    h^*(z_i) :=   
\begin{cases}
      1 & \text{with probability} \quad \frac{1}{2} - \epsilon \\
      0 & \text{with probability} \quad \frac{1}{2} + \epsilon 
\end{cases} 
\end{equation}
and for all $i \in \Iintv{d+1,2d}$ (independently): 
\begin{equation}
    h^*(z_i) :=   
\begin{cases}
      1 & \text{with probability} \quad \frac{1}{2} + \epsilon \\
      0 & \text{with probability} \quad \frac{1}{2} - \epsilon 
\end{cases} 
\end{equation}
\item $H_1$: picks $h^*$ such that for all $i \in \Iintv{1,d} $ independently: 
        \begin{equation}
    h^*(z_i) :=   
\begin{cases}
      1 & \text{with probability} \quad \frac{1}{2} + \epsilon \\
      0 & \text{with probability} \quad \frac{1}{2} - \epsilon 
\end{cases} 
\end{equation}
and for all $i \in \Iintv{d+1,2d}$ (independently): 
\begin{equation}
    h^*(z_i) :=   
\begin{cases}
      1 & \text{with probability} \quad \frac{1}{2} - \epsilon \\
      0 & \text{with probability} \quad \frac{1}{2} + \epsilon 
\end{cases} 
\end{equation}
\end{itemize}

If $h^*$ is chosen under hypothesis $H_i$, the probability that involves $h^*$ will be denoted $\prob_i$.

The case where $\text{VC}(\cH) \in 2\N+1$ reduces to $\text{VC}(\cH) \in 2 \N$ by giving a delta mass distribution to $\zeta_{2d+1}$ on the subspace shattered by $\cH$.

\noindent\textbf{Step 2: Bounding demographic parity by bounding $p$ and $\mathrm{Inf}_A$}

In order to get a lower bound for estimating statistical parity, we express it in terms of the probability of positives and the randomized influence function. 
    \begin{align}\label{lemma:separation}
        \prob \Big[ \hat{\mu} - \mu(h^*) > \epsilon  \Big] \geq \underbrace{\prob \Big[ \hat{p} - p(h^*)> c_{\alpha} \epsilon \Big]}_{\text{Term I}} + \underbrace{\prob \Big[ \widehat{\mathrm{Inf}_A} - \mathrm{Inf}_A(h^*)> c_{\alpha,1} \epsilon^2 + c_{\alpha,2} \Big]}_{\text{Term II}}\, ,
    \end{align}
    where $c_{\alpha} = \frac{4\alpha ( 1 - \alpha)(12 - \sqrt{6})}{11 (1 - 2 \alpha)}$, $c_{\alpha,1} = 1 + \frac{1}{2 \alpha (1 - \alpha)}$ and $c_ {\alpha,2} = \sqrt{\frac{2}{3}}\frac{1}{2 \alpha^2(1- \alpha)^2}$.


\noindent\textbf{Step 2.a: Bounding the Term I.}~Turning an estimation problem into a testing problem.
Under hypothesis $H_0$, we have:

\begin{align*}
    \prob_0 \Big[ \hat{p}  - p(h^*) \geq \frac{\epsilon}{2} \Big] &\geq \prob_0 \Big[\hat{p} \geq \frac{2 \alpha -11}{2}, p(h^*) \leq \frac{2 \alpha -1}{2} - \frac{\epsilon}{2} \Big] \\
    &\geq \prob_0 \Big[\hat{p} \geq \frac{2 \alpha -1}{2}\Big] +  \prob_0 \Big[p(h^*) \leq \frac{2 \alpha -1}{2} - \frac{\epsilon}{2} \Big] - 1\\
\end{align*}

Under hypothesis $H_1$, we have:

\begin{align*}
        \prob_1 \Big[ \hat{p}  - p(h^*)  \leq \frac{\epsilon}{2} \Big] &\geq \prob_1 \Big[\hat{p} \geq \frac{2 \alpha -1}{2}, p(h^*) \geq \frac{2 \alpha -1}{2} + \frac{\epsilon}{2} \Big] \\
    &\geq \prob_1 \Big[\hat{p} < \frac{2 \alpha -1}{2}\Big] +  \prob_1 \Big[p(h^*) \geq \frac{2 \alpha -1}{2} + \frac{\epsilon}{2} \Big] - 1\\
\end{align*}

Since

\begin{align*}
    \prob \Big[ \hat{p}  - p(h^*)  \geq \frac{\epsilon}{2} \Big] = \frac{1}{2} \prob_0 \Big[\hat{p}  - p(h^*)  \geq \frac{\epsilon}{2} \Big] + \frac{1}{2} \prob_1 \Big[ \hat{p}  - p(h^*)  \geq \frac{\epsilon}{2} \Big]
\end{align*}

Using the fact that $\prob[A \cap B] \geq \prob[A] + \prob[ B] -1$, we have:

\begin{align}
     \prob \Big[ \hat{p}  - p(h^*)   \geq \frac{\epsilon}{2} \Big] \geq \frac{1}{2}  &\Biggl(\prob_0 \Big[\hat{p} \geq \frac{2 \alpha -1}{2}\Big] + \prob_1 \Big[\hat{p} < \frac{2 \alpha -1}{2}\Big] \label{part1}\\
      &+ \prob_0 \Big[p(h^*) \leq \frac{2 \alpha -1}{2} - \frac{\epsilon}{2} \Big] +  \prob_1 \Big[p(h^*) \geq \frac{2 \alpha -1}{2} + \frac{\epsilon}{2} \Big] - 2\Biggl) \label{part2}
\end{align}

By Le Cam's lemma:

\begin{align}
    \prob_0 \Big[\hat{p} \geq \frac{2 \alpha -1}{2}\Big] + \prob_1 \Big[\hat{p} < \frac{2 \alpha -1}{2}\Big] \geq 1 - \TV{\mathbb{P}_0}{\mathbb{P}_1}\label{LeCamp}
\end{align}

\paragraph{Concentration of  $p(h^*)$.}

To lower bound the remaining term in \eqref{part2}, we prove Lemma~\ref{lemma:conc_true_prob}.

This proves result~\ref{pP0UB}.

Similar to the proof of the first result, by Hoeffding inequality,
  \begin{align*}
        \prob_1 \Big[ p^{+}(h^*) > \frac{\alpha}{2} - \frac{\epsilon}{2} \Big] &\leq 2 \exp{\Big(- \frac{d \epsilon^2}{2 \alpha^2}\Big)}\\
        \prob_1 \Big[ p^{-}(h^*) >  \frac{1 - \alpha}{2} - \frac{\epsilon}{2} \Big] &\leq 2 \exp{\Big(- \frac{d \epsilon^2}{2 (1-\alpha)^2}\Big)}
    \end{align*}
    The proof of result~\ref{pP1UB} concludes by proceeding with the remaining steps in the same manner as the previous proof.

    \begin{align}\label{pP0UB}
        \prob_0 \Big[ p(h^*) \leq \frac{2\alpha - 1}{2} - \frac{\epsilon}{2} \Big] \geq 1 - 2 \exp{\Big(- \frac{d \epsilon^2}{32 \alpha^2}\Big)} - 2 \exp{\Big(- \frac{d \epsilon^2}{2 (1-\alpha)^2}\Big)}
    \end{align}
        \begin{align}\label{pP1UB}
        \prob_1\Big[p(h^*) \geq \frac{2\alpha - 1}{2} + \frac{\epsilon}{2}\Big] \geq 1 - 2 \exp{\Big(- \frac{d \epsilon^2}{32 \alpha^2}\Big)} - 2 \exp{\Big(- \frac{d \epsilon^2}{2 (1-\alpha)^2}\Big)}
    \end{align}

By symmetry of the statistical test we have 
the result in ~\ref{pP1UB}.

\paragraph{Step 2.b: Bounding Term II.}
Similar to \textbf{step 2.a}, we have:
\begin{align*}
    \prob \Big[| \widehat{\mathrm{Inf}_A}  - \mathrm{Inf}_A(h^*)|  \geq \frac{\epsilon}{2} \Big] = \frac{1}{2} \prob_0 \Big[|\widehat{\mathrm{Inf}_A}  - \mathrm{Inf}_A(h^*)|  \geq \frac{\epsilon}{2} \Big] + \frac{1}{2} \prob_1 \Big[ |\widehat{\mathrm{Inf}_A}  - \mathrm{Inf}_A(h^*)|  \geq \frac{\epsilon}{2} \Big]
\end{align*}

We deduce

\begin{align}\label{part1}
     \prob \Big[\widehat{\mathrm{Inf}_A}  - \mathrm{Inf}_A(h^*)  \geq \frac{\epsilon}{2} \Big] \geq \frac{1}{2}  \Biggl(\prob_0 \Big[\widehat{\mathrm{Inf}_A} \geq \frac{1}{2}\Big] + \prob_1 \Big[\widehat{\mathrm{Inf}_A} < \frac{1}{2}\Big] +  \\
     \prob_0 \Big[\mathrm{Inf}_A(h^*) \leq \frac{1}{2} - \frac{\epsilon}{2} \Big] +  \prob_1 \Big[\mathrm{Inf}_A(h^*) \geq \frac{1}{2} + \frac{\epsilon}{2} \Big] - 2\Biggl) \label{part2Inf}
\end{align}

By Le Cam's lemma, we have:
\begin{align*}      
\mathbb{P}_0\bigg(\mathrm{Inf}_A(h^*)>\frac{1}{2}\bigg) + \mathbb{P}_1\bigg(\mathrm{Inf}_A(h^*) \leq \frac{1}{2}\bigg) \geq 1 - \TV{\mathbb{P}_0}{\mathbb{P}_1}
\end{align*}

\paragraph{Concentration of  $\mathrm{Inf}_A(h^*)$.}

To lower bound the remaining term in \eqref{part2Inf}, we prove Lemma~\ref{lemma:conc_inf_func}.

Under hypothesis $H_0$, we have:

\begin{align*}
    \prob_0 \Big[ \widehat{\mathrm{Inf}_A}  - \mathrm{Inf}_A(h^*) \geq \frac{\epsilon^2}{2} \Big] &\geq \prob_0 \Big[\widehat{\mathrm{Inf}_A} \geq \frac{1}{2}, \mathrm{Inf}_A(h^*) \leq \frac{1}{2} - \frac{\epsilon^2}{2} \Big] \\
    &\geq \prob_0 \Big[\widehat{\mathrm{Inf}_A} \geq \frac{1}{2}\Big] +  \prob_0 \Big[\mathrm{Inf}_A(h^*) \leq \frac{1}{2} - \frac{\epsilon^2}{2} \Big] - 1\\
\end{align*}

Under hypothesis $H_1$, we have:

\begin{align*}
        \prob_1 \Big[ \widehat{\mathrm{Inf}_A}  - \mathrm{Inf}_A(h^*)  \geq \frac{\epsilon}{2} \Big] &\geq \prob_1 \Big[\widehat{\mathrm{Inf}_A} \geq \frac{1}{2}, \mathrm{Inf}_A(h^*) \geq \frac{1}{2} - \frac{\epsilon}{2} \Big] \\
    &\geq \prob_1 \Big[\widehat{\mathrm{Inf}_A} < \frac{1}{2}\Big] +  \prob_1 \Big[\mathrm{Inf}_A(h^*) \geq \frac{1}{2} - \frac{\epsilon}{2} \Big] - 1\\
\end{align*}

\paragraph{Step 3: Upper bounding the statistical distances}

Let's show that $H_0$ and $H_1$ are hard to distinguish. In other words, let's show that $\mathcal{D}_{KL}(\mathbb{P}_0||\mathbb{P}_1) = \mathcal{O}(\epsilon^2)$ 

The quantity $\mathcal{D}_{KL}\Bigl(\mathbb{P}_0||\mathbb{P}_1\Bigl)$ depends on how the algorithm $\mathcal{A}$ interacts with the oracle $\mathcal{O}(h^*)$ and construct a brick of history denoted by $\mathcal{H}^{ist}$. We can observe that this quantity is exactly $\mathcal{D}_{KL}\Bigl(\mathbb{P}_0(y|(x,y) \in \mathcal{H}^{ist},x)||\mathbb{P}_1(y|(x,y) \in \mathcal{H}^{ist},x)\Bigl)$ averaged on the whole available querying set. More formally, we prove Lemma~\ref{lemma2} that states
    $$\mathcal{D}_{KL}\Bigl(\mathbb{P}_0 || \mathbb{P}_1\Bigl) = \sum_{i=1}^N \mathbb{E} \Biggl[ \mathcal{D}_{KL}\bigg(\mathbb{P}_0(y_i|(x,y) \in \mathcal{H}^{ist}_{i-1},x_i) \bigg|\bigg| \mathbb{P}_1(y_i|(x,y) \in \mathcal{H}^{ist}_{i-1},x_i)\bigg)     \Biggl].$$

The next step is to upper bound this quantity:
At iteration I, we distinguish between two separate cases:

\begin{itemize}
    \item If $x_i \in \mathcal{H}^{ist}_{i-1}$, then $\mathcal{A}$ will always output the same value under both hypotheses $H_0$ and $H_1$, which was sent by oracle $\mathcal{O}(h^*)$.
    Hence, $$\kldiv{\mathbb{P}_0(y_i|(x,y) \in \mathcal{H}^{ist}_{i-1},x_i)}{\mathbb{P}_1(y_i|(x,y) \in \mathcal{H}^{ist}_{i-1},x_i)} = 0$$
    \item If $x_i \notin \mathcal{H}^{ist}_{i-1}$, we have the following table that summarizes all possibilities under hypotheses $H_0$ and $H_1$, conditioning on $\cX^{+}$:
   \begin{center}
   \begin{tabular}{ l | c  r  }
     \hline
     $H \backslash y$ & 1 & 0 \\ \hline
     $H_0$ & $\frac{1}{2} - \frac{\epsilon}{2}$ & $\frac{1}{2} + \frac{\epsilon}{2}$ \\ \hline
     $H_1$ & $\frac{1}{2} + \frac{\epsilon}{2}$ & $\frac{1}{2} - \frac{\epsilon}{2}$ \\
     \hline
   \end{tabular}
 \end{center}

 And under hypotheses $H_0$ and $H_1$, conditioning on $\cX^{-}$:
   \begin{center}
   \begin{tabular}{ l | c  r  }
     \hline
     $H \backslash y$ & 1 & 0 \\ \hline
     $H_0$ & $\frac{1}{2} + \frac{\epsilon}{2}$ & $\frac{1}{2} - \frac{\epsilon}{2}$ \\ \hline
     $H_1$ & $\frac{1}{2} - \frac{\epsilon}{2}$ & $\frac{1}{2} + \frac{\epsilon}{2}$ \\
     \hline
   \end{tabular}
 \end{center}

 From the two tables, we deduce the overall result by expanding over each protected group (e.g, $\cX^{-}, \cX^{+}$)

   \begin{center}
   \begin{tabular}{ l | c  r  }
     \hline
     $H \backslash y$ & 1 & 0 \\ \hline
     $H_0$ & $\frac{1}{2} + \frac{(1 - 2 \alpha)\epsilon}{2}$ &$\frac{1}{2} - \frac{(1 - 2 \alpha)\epsilon}{2}$ \\ \hline
     $H_1$ & $\frac{1}{2} - \frac{(1 - 2 \alpha)\epsilon}{2}$ &$\frac{1}{2} + \frac{(1 - 2 \alpha)\epsilon}{2}$ \\
     \hline
   \end{tabular}
 \end{center} 
 
 We end up with a binary entropy upper bound: 
 $$\mathcal{D}_{KL}\bigg(\mathbb{P}_0(y_i|(x,y) \in \mathcal{H}^{ist}_{i-1},x_i) \bigg|\bigg| \mathbb{P}_1(y_i|(x,y) \in \mathcal{H}^{ist}_{i-1},x_i)\bigg) = kl\bigg(\frac{1}{2} + \frac{(1 - 2 \alpha)\epsilon}{2}, \frac{1}{2} - \frac{(1 - 2 \alpha)\epsilon}{2}\bigg)$$
 \begin{fact}
     For $a,b \in (\frac{1}{4},\frac{3}{4}): \cD_\mathrm{KL}(a,b) \leq 3 (b-a)^2$
 \end{fact}

 Hence, $$\mathcal{D}_{KL}\bigg(\mathbb{P}_0(y_i|(x,y) \in \mathcal{H}^{ist}_{i-1},x_i) \bigg|\bigg| \mathbb{P}_1(y_i|(x,y) \in \mathcal{H}^{ist}_{i-1},x_i)\bigg) \leq 3 (1 - 2 \alpha)^2 \epsilon^2$$
\end{itemize}


\begin{equation}
\label{statdist}
\mathcal{D}_{KL}(\mathbb{P}_0 || \mathbb{P}_1 ) \leq 3 N (1 - 2 \alpha)^2 \epsilon^2    
\end{equation}

By Pinsker's inequality; 

\[
\begin{cases} 
    \begin{split}
        \mathbb{P}_0\bigg(p(h^*)>\frac{1}{2}\bigg) + \mathbb{P}_1\bigg(p(h^*) \leq \frac{1}{2}\bigg) \geq 1 - \sqrt{\frac{1}{2} \mathcal{D}_{KL}( \mathbb{P}_0||\mathbb{P}_1)}
    \end{split} \\
    \begin{split}
\mathbb{P}_0\bigg(\mathrm{Inf}_A(h^*)>\frac{1}{2}\bigg) + \mathbb{P}_1\bigg(\mathrm{Inf}_A(h^*) \leq \frac{1}{2}\bigg) \geq 1 - \sqrt{\frac{1}{2} \mathcal{D}_{KL}( \mathbb{P}_0||\mathbb{P}_1)}\label{LeCamInf}
    \end{split}
\end{cases}
\]

By using result from \eqref{statdist}, 

\[
\begin{cases} 
    \begin{split}
        \mathbb{P}_0\bigg(p(h^*)>\frac{1}{2}\bigg) + \mathbb{P}_1\bigg(p(h^*) \leq \frac{1}{2}\bigg) \geq 1 - \sqrt{\frac{3 N (1 - 2 \alpha)^2 \epsilon^2  }{2}}
    \end{split} \\
    \begin{split}
\mathbb{P}_0\bigg(\mathrm{Inf}_A(h^*)>\frac{1}{2}\bigg) + \mathbb{P}_1\bigg(\mathrm{Inf}_A(h^*) \leq \frac{1}{2}\bigg) \geq 1 - \sqrt{\frac{3 N (1 - 2 \alpha)^2 \epsilon^2  }{2}}
        \label{LeCamInf}
    \end{split}
\end{cases}
\]

Results in \eqref{pP0UB} and \eqref{pP1UB} further yield 

\begin{align*}
    \prob \Big[ \hat{p}  - p(h^*)   \geq \frac{\epsilon}{2} \Big] &\geq \frac{1}{2} - 2 \exp{\Big(- \frac{d \epsilon^2}{32 \alpha^2}\Big)} - 2 \exp{\Big(- \frac{d \epsilon^2}{2 (1-\alpha)^2}\Big)} - \sqrt{\frac{3 N   }{2}} \frac{|1 - 2 \alpha| \epsilon}{2}\\
     & \geq \frac{1}{2} - 4 \exp{\Big(- \frac{d \epsilon^2 }{8 M_{\alpha}^2}\Big)}  - \sqrt{\frac{3 N   }{2}} \frac{|1 - 2 \alpha| \epsilon}{2}\,,
\end{align*}
where $M_{\alpha} = \max(\alpha, 1- \alpha)$.

Further, \eqref{H0Inf} and \eqref{H1Inf} yield
\begin{align*}
    \prob \Big[ \widehat{\mathrm{Inf}_A}  - \mathrm{Inf}_A(h^*)  \geq \frac{\epsilon}{2} \Big] &\geq \frac{5}{2} - 4 \exp{\Big(- \frac{d \epsilon}{2}\Big)} - 4 \exp{\Big(- \frac{d \epsilon}{18}\Big)} - \sqrt{\frac{3 N (1 - 2 \alpha)^2 \epsilon^2  }{8}} \\
    &\geq \frac{5}{2} - 8 \exp{\Big(- \frac{d \epsilon}{18}\Big)} - \sqrt{\frac{3 N (1 - 2 \alpha)^2 \epsilon^2  }{8}} 
\end{align*}

Finally, solving the inequality
$$
    3 -4 \exp{\frac{-d \epsilon^2}{18}} - \sqrt{\frac{3 N (1 - 2 \alpha)^2 \epsilon^2  }{8}}  \geq \delta
$$ 
yields the sample complexity to be
$N \leq \frac{8}{3 (1 - 2 \alpha)^2 \epsilon^2} \Biggl(\delta - 3 + 4 \exp{(-\frac{d \epsilon^2}{18})}\Biggl)^2.$
\end{proof}




\clearpage
\subsection{Additional Technical Lemmas}

\begin{lemma}\label{lemma:conc_true_prob}
    \begin{align}\label{pP0UB}
        \prob_0 \Big[ p(h^*) \leq \frac{2\alpha - 1}{2} - \frac{\epsilon}{2} \Big] \geq 1 - 2 \exp{\Big(- \frac{d \epsilon^2}{32 \alpha^2}\Big)} - 2 \exp{\Big(- \frac{d \epsilon^2}{2 (1-\alpha)^2}\Big)}
    \end{align}

        \begin{align}\label{pP1UB}
        \prob_1\Big[p(h^*) \geq \frac{2\alpha - 1}{2} + \frac{\epsilon}{2}\Big] \geq 1 - 2 \exp{\Big(- \frac{d \epsilon^2}{32 \alpha^2}\Big)} - 2 \exp{\Big(- \frac{d \epsilon^2}{2 (1-\alpha)^2}\Big)}
    \end{align}
\end{lemma}

\begin{proof}
    \begin{align*}
        p(h^*) &= \prob\Big[h^*(x) = 1\Big]  \\
        &= \alpha \prob\Big[h^*(x) = 1\Big| \cX^{+}\Big] + (1-\alpha) \prob\Big[h^*(x) = 1\Big| \cX^{-}\Big]\\
        &= \frac{\alpha}{d} \sum_{i=1}^d \mathds{1}_{\{h^*(z_i) = 1\}} + \frac{1-\alpha}{d} \sum_{i=1}^d \mathds{1}_{\{h^*(z_{d+i}) = 1\}}\\
        p(h^*) &= p^{+}(h^*) + p^{-}(h^*)
    \end{align*}

    Where $p^{+}(h^*) = \frac{\alpha}{d} \sum_{i=1}^d \mathds{1}_{\{h^*(z_i) = 1\}}$ and $p^{-}(h^*) = \frac{1-\alpha}{d} \sum_{i=1}^d \mathds{1}_{\{h^*(z_{d+i}) = 1\}} $

    Under $H_0$ (resp. $H_1$), $\frac{d}{\alpha} p^{+}(h^*)$ is the sum of $d$ Bernoulli variables of mean $\frac{1}{2} - \epsilon$ (resp. $\frac{1}{2} + \epsilon$).  Under $H_0$ (resp. $H_1$), $\frac{d}{1-\alpha} p^{+}(h^*)$ is the sum of $d$ Bernoulli variables of mean $\frac{1}{2} + \epsilon$ (resp. $\frac{1}{2} - \epsilon$).

    \begin{align*}
        \prob_0 \Big[ p^{+}(h^*) > \frac{\alpha}{2} - \frac{\epsilon}{4} \Big] &\leq 2 \exp{\Big(- \frac{d \epsilon^2}{32 \alpha^2}\Big)}\\
        \prob_0 \Big[ p^{-}(h^*) > \frac{\epsilon}{2} - \frac{1 - \alpha}{2} \Big] &\leq 2 \exp{\Big(- \frac{d \epsilon^2}{2 (1-\alpha)^2}\Big)}
    \end{align*}

On the other hand, 

\begin{align*}
    \prob_0 \Big[ p(h^*) \leq \frac{2\alpha - 1}{2} - \frac{\epsilon}{2} \Big]  &\geq   \prob_0 \Big[ p^{+}(h^*) \leq \frac{\alpha}{2} - \frac{\epsilon}{4} , p^{-}(h^*) \leq \frac{\epsilon}{2} - \frac{1 - \alpha}{2} \Big]\\
    &\geq   \prob_0 \Big[ p^{+}(h^*) \leq \frac{\alpha}{2} - \frac{\epsilon}{4}\Big] + \prob_0 \Big[p^{-}(h^*) \leq \frac{\epsilon}{2} - \frac{1 - \alpha}{2} \Big] - 1\\
    & \geq 1 - 2 \exp{\Big(- \frac{d \epsilon^2}{32 \alpha^2}\Big)} - 2 \exp{\Big(- \frac{d \epsilon^2}{2 (1-\alpha)^2}\Big)}
\end{align*}   
\end{proof}

\begin{lemma}[Concentration of Influence Function]\label{lemma:conc_inf_func}
    \begin{equation}\label{H0Inf}
        \mathbb{P}_0 \Big[\mathrm{Inf}_A(h^*) \leq \frac{1+\epsilon}{2}   \Big] \geq 3 -4\exp{\Big(- \frac{d \epsilon}{2 }\Big)} - 4\exp{\Big(- \frac{d \epsilon}{18}\Big)}
    \end{equation}

 \begin{equation}\label{H1Inf}
        \mathbb{P}_1 \Big[\mathrm{Inf}_A(h^*) > \frac{1 - \epsilon}{2}   \Big] \geq 3 -4\exp{\Big(- \frac{d \epsilon}{2 }\Big)} - 4\exp{\Big(- \frac{d \epsilon}{18}\Big)}
    \end{equation}
\end{lemma}

\begin{proof}    

\begin{align*}
    \mathrm{Inf}_A(h^*) & = \prob \Big[h^*(x) \neq h^*(x') \Big| x \in \cX^{+}, x' \in \cX^{-}  \Big]\\
    &= \prob \Big[h^*(x) =1 , h^*(x') =0 \Big| x \in \cX^{+}, x' \in \cX^{-}  \Big] + \prob \Big[h^*(x) = 0, h^*(x')=1 \Big| x \in \cX^{+}, x' \in \cX^{-}  \Big] \\
    & = \prob \Big[h^*(x) =1  \Big| x \in \cX^{+}  \Big] \prob \Big[h^*(x) =0 \Big|x \in \cX^{-}  \Big] + \prob \Big[h^*(x) = 0 \Big| x \in \cX^{+} \Big] \prob \Big[ h^*(x)=1 \Big| x \in \cX^{-}  \Big]\\
    &= \frac{1}{d^2} \underset{1 \leq i,j \leq d}{\sum} \mathds{1}_{\{h^*(z_i) =1 \}} \mathds{1}_{\{h^*(z_{d+j}) =0  \}} + \frac{1}{d^2} \underset{1 \leq i,j \leq d}{\sum} \mathds{1}_{\{h^*(z_i) =0 \}} \mathds{1}_{\{h^*(z_{d+j}) =1  \}}\\
    \mathrm{Inf}_A(h^*) &= \mathrm{Inf}_{A,1}^{+}(h^*) \mathrm{Inf}_{A,0}^{-}(h^*) + \mathrm{Inf}_{A,0}^{+}(h^*) \mathrm{Inf}_{A,1}^{-}(h^*)
\end{align*}

Where,
$ \mathrm{Inf}_{A,1}^{+}(h^*) =  \frac{1}{d} \sum_{i=1}^d \mathds{1}_{\{h^*(z_i) =1 \}}$

$\mathrm{Inf}_{A,0}^{-}(h^*) = \frac{1}{d} \sum_{i=1}^d \mathds{1}_{\{h^*(z_{d+i}) =0 \}}$, 

$\mathrm{Inf}_{A,0}^{+}(h^*) = \frac{1}{d} \sum_{i=1}^d \mathds{1}_{\{h^*(z_i) =0 \}}$, 

$\mathrm{Inf}_{A,1}^{-}(h^*)=  \frac{1}{d} \sum_{i=1}^d \mathds{1}_{\{h^*(z_{d+i}) =1 \}}$.

\begin{itemize}
    \item    Under $H_0$ (resp. $H_1$), $\mathrm{Inf}_{A,1}^{+}(h^*))$ is the sum of $d$ Bernoulli variables of mean $\frac{1}{2} - \epsilon$ (resp. $\frac{1}{2} + \epsilon$). 
    \item Under $H_0$ (resp. $H_1$), $\mathrm{Inf}_{A,0}^{-}(h^*)$ is the sum of $d$ Bernoulli variables of mean $\frac{1}{2} - \epsilon$ (resp. $\frac{1}{2} + \epsilon$).
    \item Under $H_0$ (resp. $H_1$), $\mathrm{Inf}_{A,0}^{+}(h^*)$ is the sum of $d$ Bernoulli variables of mean $\frac{1}{2} + \epsilon$ (resp. $\frac{1}{2} - \epsilon$).
    \item Under $H_0$ (resp. $H_1$), $\mathrm{Inf}_{A,1}^{-}(h^*)$ is the sum of $d$ Bernoulli variables of mean $\frac{1}{2} + \epsilon$ (resp. $\frac{1}{2} - \epsilon$).
\end{itemize}

Applying Hoeffding inequality under hypothesis $H_0$ gives:

\begin{align}\label{plus1}
    \prob_0 \Big[\mathrm{Inf}_{A,1}^{+}(h^*) > \frac{1}{2} - \frac{\epsilon}{2} \Big] \leq 2 \exp{\Big(- \frac{d \epsilon^2}{2}\Big)}
\end{align}

\begin{align}\label{minus0}
    \prob_0 \Big[\mathrm{Inf}_{A,0}^{-}(h^*) > \frac{1}{2} - \frac{\epsilon}{2} \Big] \leq 2 \exp{\Big(- \frac{d \epsilon^2}{2}\Big)}
\end{align}

From \ref{plus1} and \ref{minus0}, we deduce:
\begin{align}\label{res1plus0minus}
    \prob_0 \Big[\mathrm{Inf}_{A,1}^{+}(h^*)\mathrm{Inf}_{A,0}^{-}(h^*) \leq \Big(\frac{1}{2} - \frac{\epsilon}{2} \Big)^2 \Big] \geq 2 - 4 \exp{\Big(- \frac{d \epsilon^2}{2}}\Big) 
\end{align}
Similar, the upper bound of the second part is: 

\begin{align}\label{res0plus1minus}
    \prob_0 \Big[\mathrm{Inf}_{A,0}^{+}(h^*) \mathrm{Inf}_{A,1}^{-}(h^*) \leq \Big(\frac{1}{2} + \frac{\epsilon}{2} \Big)^2 \Big] \geq 2 - 4 \exp{\Big(- \frac{d \epsilon^2}{18}}\Big)  
\end{align}

Combining results \ref{res1plus0minus} and\ref{res0plus1minus} yields result~\ref{H0Inf}. By the symmetry of the hypotheses $H_0$ and $H_1$, we obtain the second result.
\end{proof}

\begin{lemma}\label{lemma2}
    $$\mathcal{D}_{KL}\Bigl(\mathbb{P}_0 || \mathbb{P}_1\Bigl) = \sum_{i=1}^N \mathbb{E} \Biggl[ \mathcal{D}_{KL}\bigg(\mathbb{P}_0(y_i|(x,y) \in \mathcal{H}^{ist}_{i-1},x_i) \bigg|\bigg| \mathbb{P}_1(y_i|(x,y) \in \mathcal{H}^{ist}_{i-1},x_i)\bigg)     \Biggl]$$
\end{lemma}

\begin{proof}
By definition, 

\begin{equation*}
    \begin{split}
    &~~~\mathcal{D}_{KL}(\mathbb{P}_0 || \mathbb{P}_1) = \sum_{\mathcal{Q} \in \mathcal{H}^{ist}_N}  \mathbb{P}_0(\mathcal{Q}) \, \log \frac{\mathbb{P}_0(\mathcal{Q})}{\mathbb{P}_1(\mathcal{Q})} \\
    & = \underset{\substack{\mathcal{Q} \in \mathcal{H}^{ist}_N \\ \mathcal{Q} = \{(x_1,y_1), \dots (x_N,y_N)\}}}{\sum}  \mathbb{P}_0(\mathcal{Q}) \, \log \frac{\prod_{i=1}^N \mathbb{P}_0(y_i|(x,y) \in \mathcal{H}^{ist}_{i-1},x_i) \mathbb{P}_{\mathcal{A}}(x_i|(x,y) \in \mathcal{H}^{ist}_{i-1})}{\prod_{i=1}^N  \mathbb{P}_1((y_i|(x,y) \in \mathcal{H}^{ist}_{i-1},x_i) \mathbb{P}_{\mathcal{A}}(x_i|(x,y) \in \mathcal{H}^{ist}_{i-1})} \\   
    & = \underset{\substack{\mathcal{Q} \in \mathcal{H}^{ist}_N \\ \mathcal{Q} = \{(x_1,y_1), \dots (x_N,y_N)\}}}{\sum}  \mathbb{P}_0(\mathcal{Q}) \, \sum_{i=1}^N \log \frac{ \mathbb{P}_0(y_i|(x,y) \in \mathcal{H}^{ist}_{i-1},x_i)}{\mathbb{P}_1((y_i|(x,y) \in \mathcal{H}^{ist}_{i-1},x_i)} \\
    & = \sum_{i=1}^N \underset{\substack{\mathcal{Q} \in \mathcal{H}^{ist}_N \\ \mathcal{Q} = \{(x_1,y_1), \dots (x_i,y_i)\}}}{\sum}  \mathbb{P}_0(\mathcal{Q}) \,  \log \frac{ \mathbb{P}_0(y_i|(x,y) \in \mathcal{H}^{ist}_{i-1},x_i)}{\mathbb{P}_1((y_i|(x,y) \in \mathcal{H}^{ist}_{i-1},x_i)} \\
    & = \sum_{i=1}^N \underset{\{(x_1,y_1), \dots (x_i,y_i)\}}{\sum}  \mathbb{P}_0(y_i|(x,y) \in \mathcal{H}^{ist}_{i-1},x_i) \mathbb{P}_0((x,y) \in \mathcal{H}^{ist}_{i-1},x_i) \,  \log \frac{ \mathbb{P}_0(y_i|(x,y) \in \mathcal{H}^{ist}_{i-1},x_i)}{\mathbb{P}_1((y_i|(x,y) \in \mathcal{H}^{ist}_{i-1},x_i)} \\
    & = \sum_{i=1}^N \underset{\{(x_1,y_1), \dots (x_{i-1},y_{i-1}),x_i\}}{\sum}  \mathbb{P}_0((x,y) \in \mathcal{H}^{ist}_{i-1},x_i) \sum_{y_i} 
 \mathbb{P}_0(y_i|(x,y) \in \mathcal{H}^{ist}_{i-1},x_i)  \,  \log \frac{ \mathbb{P}_0(y_i|(x,y) \in \mathcal{H}^{ist}_{i-1},x_i)}{\mathbb{P}_1((y_i|(x,y) \in \mathcal{H}^{ist}_{i-1},x_i)} \\ 
    & = \sum_{i=1}^N \sum_{\bH_{i-1},x_i} \mathbb{P}_0((x,y) \in \bH_{i-1},x_i)  \mathcal{D}_{KL}\Bigg(\mathbb{P}_0(y_i|(x,y) \in \bH_{i-1},x_i) \bigg|\bigg| \mathbb{P}_1(y_i|(x,y) \in \bH_{i-1},x_i) \Bigg)\\
    \end{split}
\end{equation*}

Hence, $$\mathcal{D}_{KL}\Bigl(\mathbb{P}_0 || \mathbb{P}_1\Bigl) = \sum_{i=1}^N \mathbb{E} \Biggl[ \mathcal{D}_{KL}\bigg(\mathbb{P}_0(y_i|(x,y) \in \bH_{i-1},x_i) \bigg|\bigg| \mathbb{P}_1(y_i|(x,y) \in \bH_{i-1},x_i)\bigg)     \Biggl]$$
\end{proof}

\section{Extensions to Multi-class Classification}\label{sec:multiclass}
If $\{a_1, \cdots, a_n  \}$ denotes the set of categories such that for all $i \neq j \in \{1, \cdots, n\}, \cX_{i,j} = h^{-1} (\{a_i, a_j\})$, and $A_h$ the set: 
\begin{align*}
    \cA_h = \bigcup_{i \neq j} \Bigl\{h_{i,j}: \cX_{i,j} \to \{a_i,a_j\}, h_{i,j}(\cX_{i,j} ) =   h(\cX_{i,j} )  \Bigl\}
\end{align*}

Based on the result in Proposition~\ref{groupfourier} the Fourier pattern of multicalibration is as follows:

$$\robust{h} = \max_{g \in \cA_h} P_{\hat{g}}^{-1}(0)$$

This adaptation is evaluated empirically to assess how well \texttt{AFA} performs in this setting.
\section{Experimental Details}

All our computations are performed on an 11th Gen Intel® Core™ i7-1185G7 processor (3.00 GHz, 8 cores) with 32.0 GiB of RAM.

\subsection{Uniformly Random Sampling (I.I.D.) estimators (\texttt{Uniform})}
\label{appAA}

Random estimators use i.i.d.\ sampling in order to estimate each distributional property. We note that group fairness estimation requires a different sampling strategy and interaction with the black-box oracle of $h$.

\paragraph{Robustness.} The true robustness is defined as:
$$\robust{h} = \underset{\substack{\mathbf{x} \sim \mathcal{D} \\ \mathbf{y} \sim N_{\rho}(\mathbf{x})}}{\mathbb{P}}[h(\mathbf{x}) \neq h(\mathbf{y})] $$
Random estimator samples i.i.d.\ points from $\cD$, which we denote as $S$. Thus, the estimator can be written as
$$\widehat{\robust{h}} = \frac{1}{|S|} \underset{\substack{\mathbf{x} \in S \\ \mathbf{y} \sim N_{\rho}(\mathbf{x})}}{\sum} \mathds{1}_{h(\mathbf{x}) \neq h(\mathbf{y})}$$

\paragraph{Individual Fairness.} Likewise, individual fairness estimation given by random estimator is: 
    
    $$\widehat{\IFair{h}}= \frac{1}{|S|} \underset{\substack{\mathbf{x} \in S \\ \mathbf{y} \sim N_{\rho,l}(\mathbf{x})}}{\sum} \mathds{1}_{h(\mathbf{x}) \neq h(\mathbf{y})}$$

\paragraph{Group Fairness.}

Let $S^{+}$ denote a set of samples from the first protected group and $S^{-}$ a set of samples from the second protected group. Group Fairness (with demographic parity measure) is defined as: 
$$\widehat{\GFair{h}} = \frac{1}{|S^{+}|} \underset{\mathbf{x} \in S^{+} }{\sum} \mathds{1}_{h(\mathbf{x}) =1} - \frac{1}{|S^{-}|} \underset{\mathbf{x} \in S^{-} }{\sum} \mathds{1}_{h(\mathbf{x}) =1}$$

\subsection{Baseline Algorithms}

We assess \afa{} on statistical parity by comparing its performance in sample complexity and running time to the methodologies investigated by \cite{Reconstruct&Audit:Yan}. In their method, auditing has an additional step: approximating the model through reconstruction before plugging in the estimator. Those methodologies use active learning algorithms for approximating the black-box model i.e, CAL algorithm~\citep{Cohn1994}, along with its variant for property active estimation $\mu$-CAL, and its randomized version.

Furthermore, efficient AFA is employed to find significant Fourier coefficients within subsets containing the protected attribute, this model forces search over within subsets containing the protected attribute. In other words, $\afa{}$ focuses on half of the buckets $2^{n-1}$ (buckets that contain the protected attribute), where $n$ is the dimension of the input space.

\subsection{Additional Experimental Results}
\paragraph{Individual fairness.}

\begin{table}[ht]
    \centering
    \caption{Estimation error for individual fairness across models and datasets. \textbf{Bold} numbers mean lower error.}
\begin{tabular}{l|cccccc}
\toprule
Dataset & \multicolumn{3}{c}{COMPAS} & \multicolumn{3}{c}{Student} \\
\cmidrule(lr){2-4}\cmidrule(lr){5-7}    
Model                             
& LR             & MLP            & RF             & LR             & MLP            & RF                        \\
\midrule
\texttt{Uniform} & 0.050          & 0.072          & 0.070          & 0.12          & 0.08          & 0.173               \\
AFA  & \textbf{0.002} & \textbf{0.035} & \textbf{0.048} & \textbf{0.079} & \textbf{0.057} & \textbf{0.050} \\
\bottomrule
\end{tabular}
\end{table}

For individual fairness, the perturbation parameter $l$ is a free parameter for which Hamming distance measures individual similarity. The parameter $l$ answers the question: \textit{What degree of similarity should the model refrain from distinguishing?} 
Hence, a good auditor would have the same performance for all possible parameter values $l$. To evaluate that, we fix $\rho = 0.30$ and compare AFA and random estimator performances for a range of values of parameter $l$. Experiment details are summarized in Table \ref{tab:l}.

\begin{table*}[ht]
  \centering
    \caption{A summary of theoretical results: This table summarizes the expression of the estimation for each property with query complexity and computational complexity. \textbf{Bold} refers to the best method.}
  \label{tab:l}
    \begin{tabular}{c|c|c}
    $l$-parameter &   AFA $\IFaire$ error &   random $\IFaire$ error\\ 
    \midrule
 $11$     &   \boldmath\textbf{$0.123$}    &   $0.267$  \\
  $10$    &    \boldmath\textbf{$0.119$}  &    $0.254$ \\
 $7$     &   \boldmath\textbf{$0.141$}    &   $ 0.244$ \\
 $5$    &   \boldmath\textbf{$0.169$}   &   $ 0.230$ \\
 $3$     &  \boldmath\textbf{$0.166$}     &   $ 0.222$ \\
    \bottomrule
    \end{tabular}
\end{table*}

As Figure \ref{fig:l} shows, AFA always outperform random estimator for the property of individual fairness for all different values of perturbation parameter $l$.

\begin{figure}[h!]
    \centering
    \subfloat[\centering $l=11$]{{\includegraphics[width=0.45\textwidth]{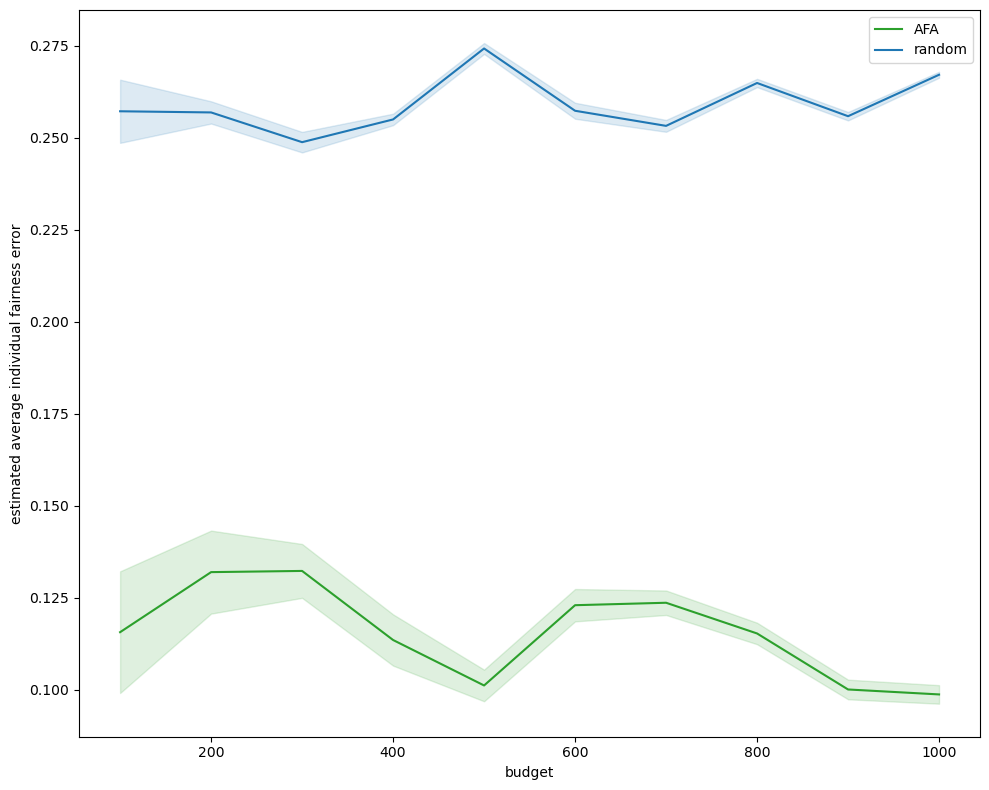} }}%
    \hfill
    \subfloat[\centering $l=10$]{{\includegraphics[width=0.45\textwidth]{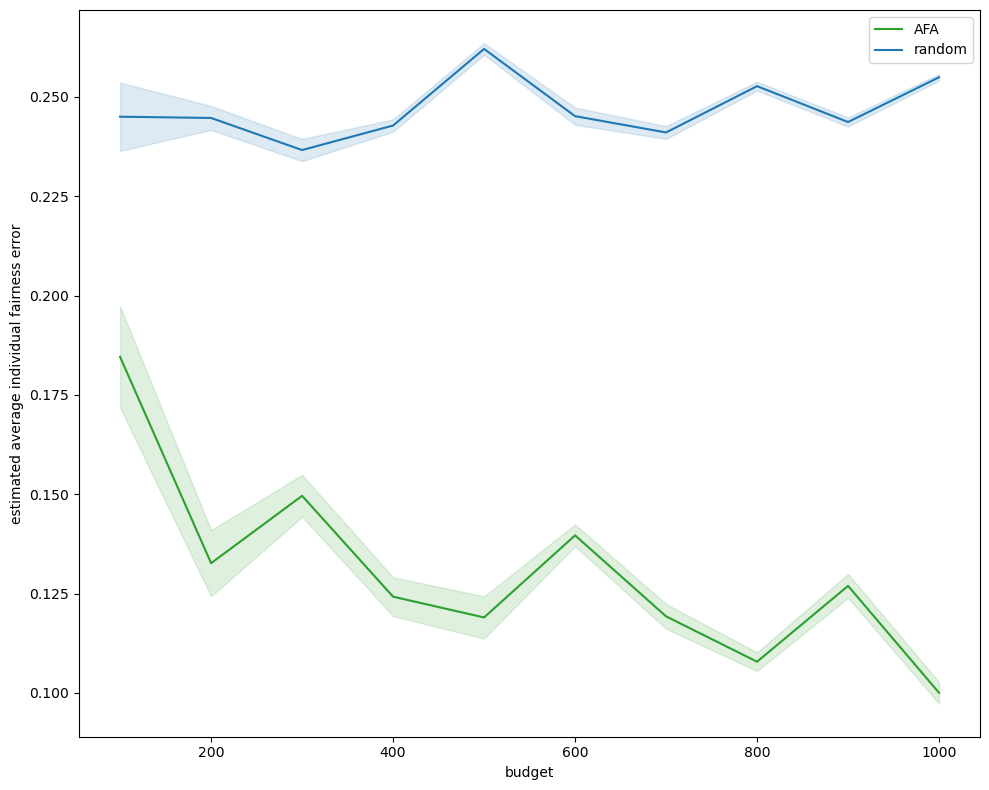} }}\\%
     \subfloat[\centering $l=7$]{{\includegraphics[width=0.45\textwidth]{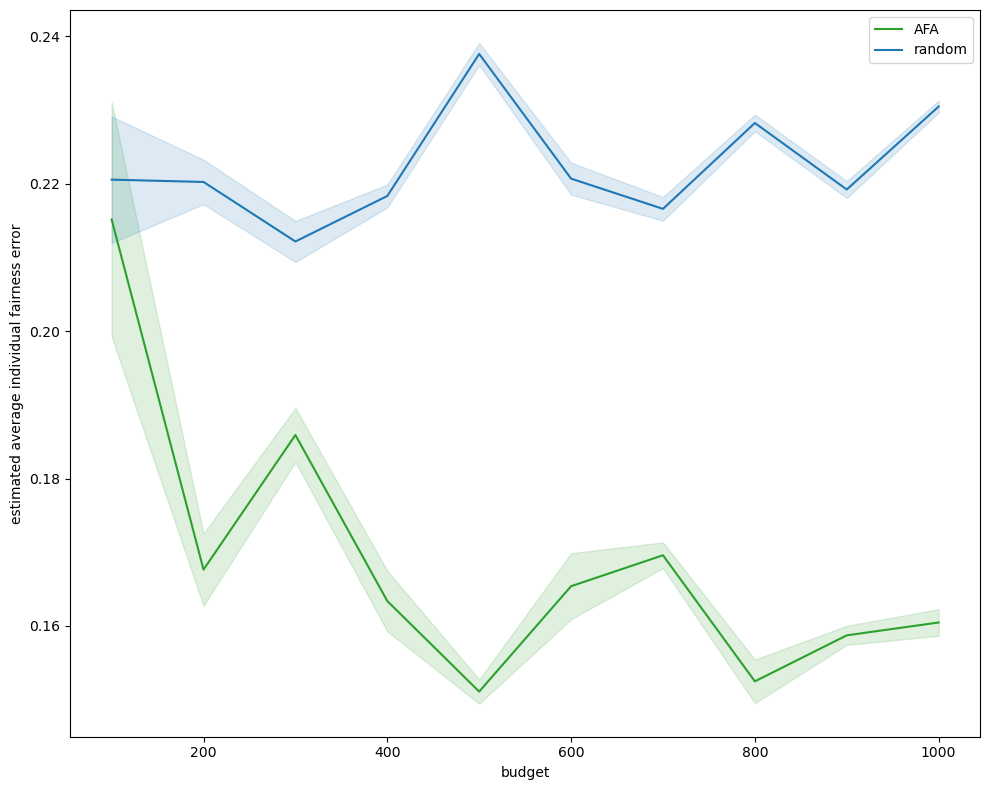} }}\hfill%
    \subfloat[\centering $l=5$]{{\includegraphics[width=0.45\textwidth]{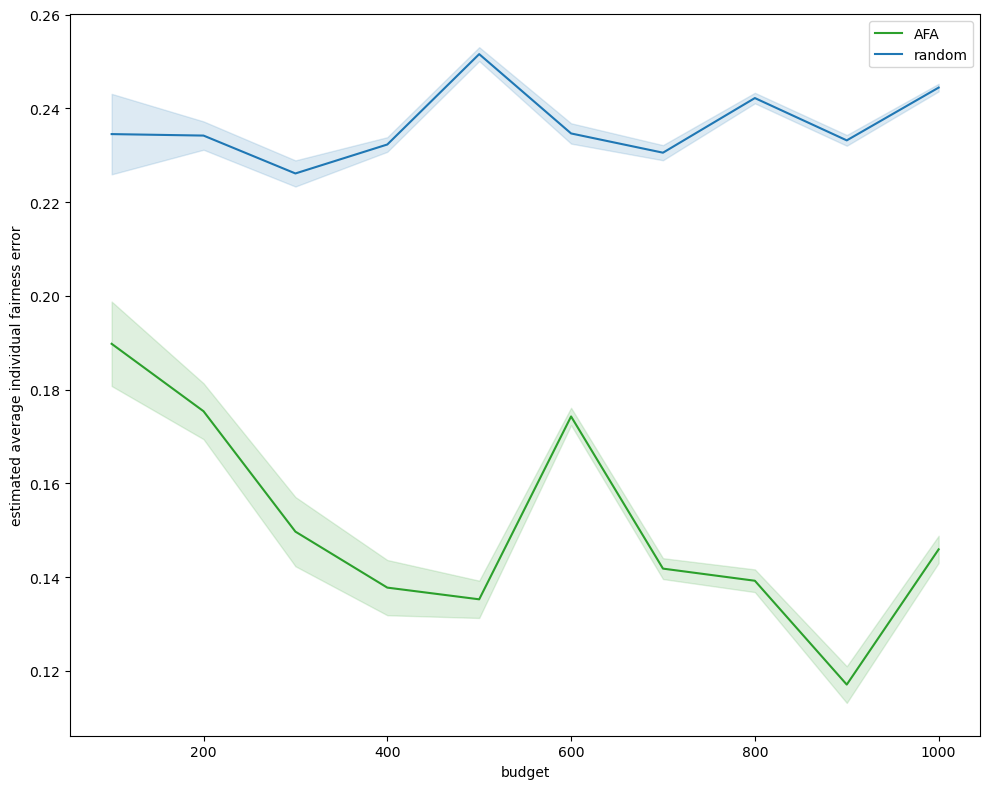} }}\\%
     \subfloat[\centering $l=3$]{{\includegraphics[width=0.45\textwidth]{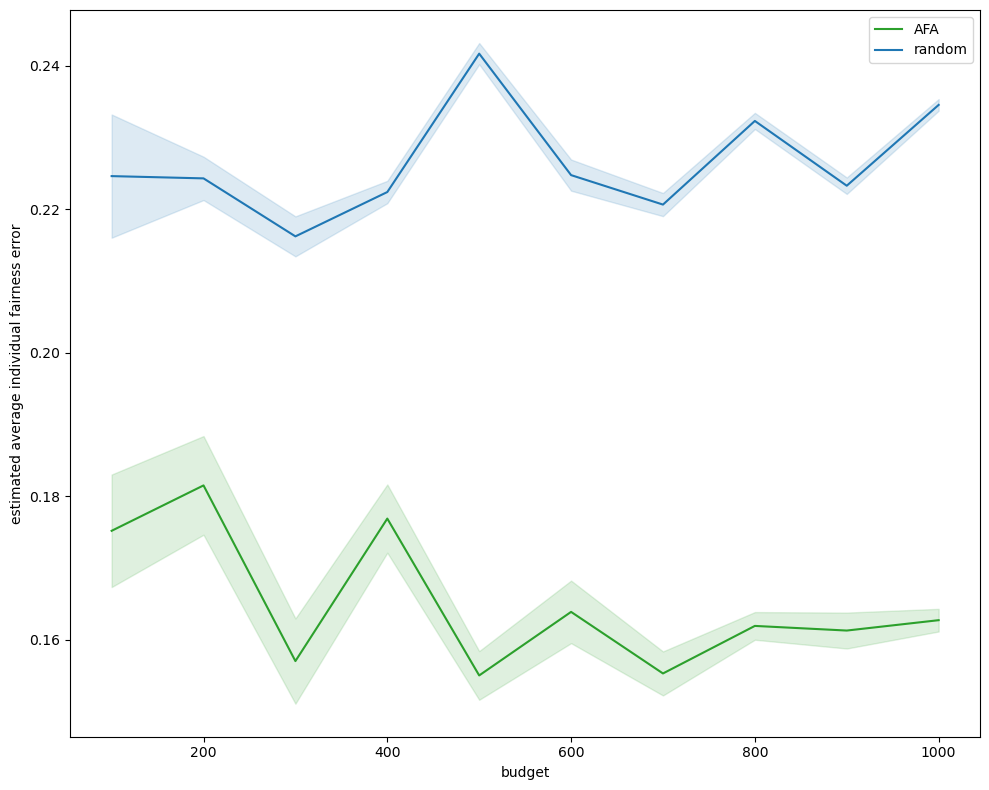} }}%
    \caption{Comparison of AFA and random estimator on COMPAS dataset for different values of perturbation parameter $l$.}%
    \label{fig:l}%
\end{figure}

\paragraph{Statistical parity.}
We evaluate SP for the Student Performance dataset, with gender as the protected attribute. Figure \ref{fig:group_fairness_1} shows that $\afa{}$'s error converges faster to the zero value compared to \texttt{Uniform}.

\begin{figure}[h!]
    \centering
    \subfloat{\includegraphics[scale = 0.5]{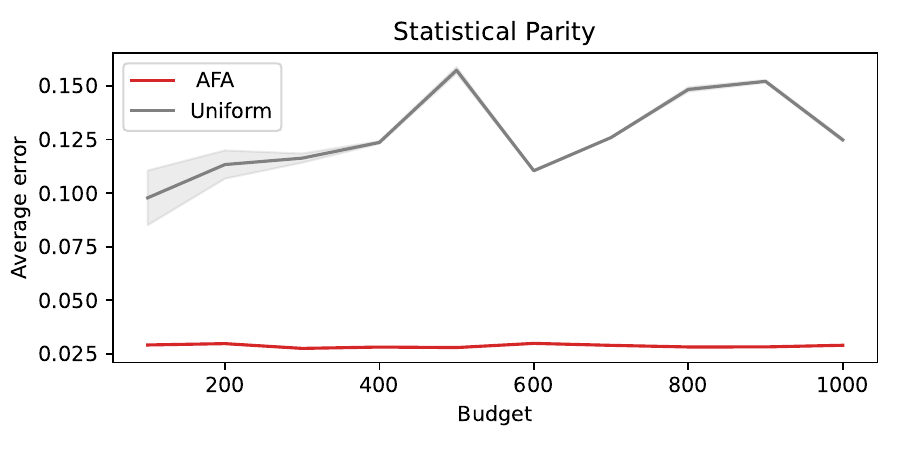}}
    \subfloat{\includegraphics[scale = 0.35]{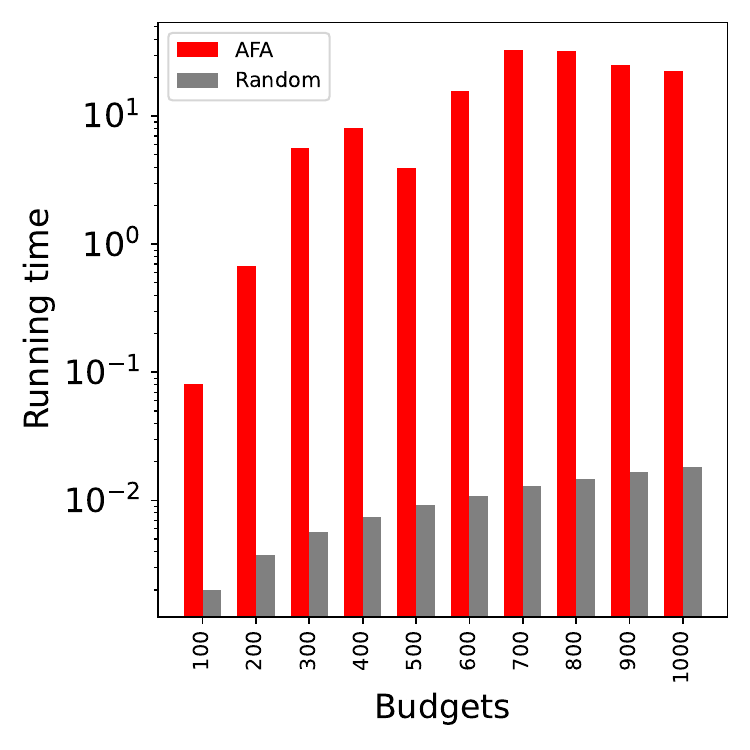}}
    \caption{Error (left) and running time (right) of auditors in estimating statistical parity of logistic regression for Student Performance dataset.}\label{fig:group_fairness_1}
\end{figure}

We empirically evaluate the Fourier Pattern for multicalibration by training a logistic regression model on the DRUG dataset, where gender is considered the protected attribute. Figure~\ref{fig:multicat_fairness} shows the consistency of {\afa{}} performance when the black-box model has multiple outcomes.

\begin{figure}[h!]
    \centering
    \subfloat{\includegraphics[scale = 0.5]{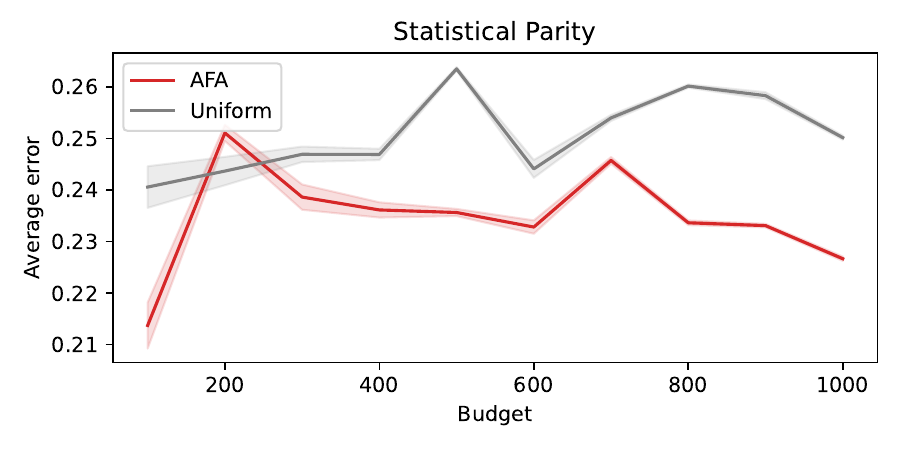}}
    \subfloat{\includegraphics[scale = 0.35]{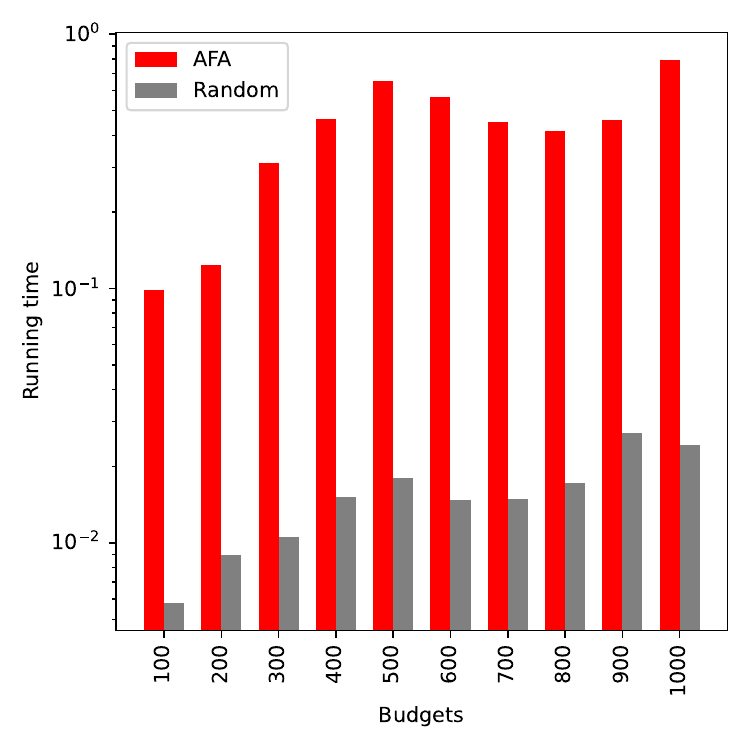}}
    \caption{Error (left) and running time (right) of different auditors in estimating statistical parity of logistic regression for Drugs Consumption dataset.}\label{fig:multicat_fairness}
\end{figure}

\end{document}

%% file: abstract.tex
  With the pervasive deployment of Machine Learning (ML) models in real-world applications, verifying and auditing properties of ML models have become a central concern.  In this work, we focus on three properties: robustness, individual fairness, and group fairness. We discuss two approaches for auditing ML model properties: estimation with and without reconstruction of the target model under audit. Though the first approach is studied in the literature, the second approach remains unexplored.
  For this purpose, we develop a new framework that quantifies different properties in terms of the Fourier coefficients of the ML model under audit but does not parametrically reconstruct it.
  We propose the Active Fourier Auditor (\afa), which queries sample points according to the Fourier coefficients of the ML model, and further estimates the properties. We derive high probability error bounds on \afa's estimates, along with the worst-case lower bounds on the sample complexity to audit them.
  Numerically we demonstrate on multiple datasets and models that \afa{} is more accurate and sample-efficient to estimate the properties of interest than the baselines.

%% file: acks.tex
This work was supported by the Regalia Project partnered by Inria and French Ministry of Finance. D. Basu acknowledges the ANR JCJC project REPUBLIC (ANR-22-CE23-0003-01) and the PEPR project FOUNDRY (ANR23-PEIA-0003). 